\documentclass[10pt,twocolumn]{article}

\usepackage[T1]{fontenc}
\usepackage[utf8]{inputenc}
\usepackage{microtype}
\usepackage{graphicx}
\usepackage{subfigure}
\usepackage{booktabs}
\graphicspath{{figures/}} 
\usepackage{mathtools}
\usepackage{bm}
\usepackage{subcaption}
\usepackage{tabularx}
\usepackage{adjustbox}
\usepackage{enumitem}
\setlist{nosep,leftmargin=*}
\usepackage{xcolor}
\usepackage{hyperref}
\usepackage[accepted]{icml2025}
\usepackage{amsmath}
\usepackage{amssymb}
\usepackage{mathtools}
\usepackage{amsthm}
\usepackage[capitalize,noabbrev]{cleveref}

\usepackage{tikz}
\usetikzlibrary{shapes.geometric,arrows,positioning,matrix,arrows.meta}
\tikzset{
  startstop/.style = {rectangle, rounded corners, minimum width=3cm, minimum height=1cm, text centered, draw=black, fill=red!30},
  process/.style = {rectangle, minimum width=3cm, minimum height=1cm, text centered, draw=black, fill=blue!30},
  arrow/.style = {thick,->,>=stealth},
  decision/.style = {diamond, minimum width=3cm, minimum height=1cm, text centered, draw=black, fill=green!30},
  state/.style = {circle, draw, minimum size=25pt, inner sep=0pt}
}

\theoremstyle{plain}
\newtheorem{theorem}{Theorem}[section]

\newtheorem{lemma}[theorem]{Lemma}
\newtheorem{corollary}[theorem]{Corollary}
\theoremstyle{definition}

\theoremstyle{remark}

\usepackage[textsize=tiny]{todonotes}

\icmltitlerunning{Mutual Information Collapse Explains Disentanglement Failure in $\beta$-VAEs}

\begin{document}

\twocolumn[
\icmltitle{Mutual Information Collapse Explains\\Disentanglement Failure in $\beta$-VAEs}

\begin{icmlauthorlist}
\icmlauthor{Minh Vu}{y1}
\icmlauthor{Xiaoliang Wan}{y2}
\icmlauthor{Shuangqing Wei}{sch}
\end{icmlauthorlist}

\icmlaffiliation{y1}{Department of Mathematics, Miami University, Oxford, OH, USA}
\icmlaffiliation{y2}{Department of Mathematics, Louisiana State University, Baton Rouge, LA, USA}
\icmlaffiliation{sch}{Division of Electrical \& Computer Engineering, Louisiana State University, Baton Rouge, LA, USA}

\icmlcorrespondingauthor{Minh Vu}{vumh2@miamioh.edu}
\icmlcorrespondingauthor{Xiaoliang Wan}{xlwan@lsu.edu}
\icmlcorrespondingauthor{Shuangqing Wei}{swei@lsu.edu}

\icmlkeywords{Machine Learning}

\vskip 0.3in
]

\printAffiliationsAndNotice{}

\begin{abstract}
The $\beta$-VAE is a foundational framework for unsupervised disentanglement, using $\beta$ to regulate the trade-off between latent factorization and reconstruction fidelity. Empirically, however, disentanglement performance exhibits a pervasive non-monotonic trend: benchmarks such as MIG and SAP typically peak at intermediate $\beta$ and collapse as regularization increases. We demonstrate that this collapse is a fundamental information-theoretic failure, where strong Kullback--Leibler pressure promotes marginal independence at the expense of the latent channel's semantic informativeness. By formalizing this mechanism in a linear-Gaussian setting, we prove that for $\beta > 1$, stationarity-induced dynamics trigger a spectral contraction of the encoder gain, driving latent-factor mutual information to zero. To resolve this, we introduce the $\lambda\beta$-VAE, which decouples regularization pressure from informational collapse via an auxiliary $L_2$ reconstruction penalty $\lambda$. Extensive experiments on dSprites, Shapes3D, and MPI3D-real confirm that $\lambda > 0$ stabilizes disentanglement and restores latent informativeness over a significantly broader range of $\beta$, providing a principled theoretical justification for dual-parameter regularization in variational inference backbones.
\end{abstract}

\section{Introduction}
\label{sec:introduction}

Disentangled representation learning seeks to map high-dimensional observations into latent variables whose coordinates correspond to distinct, independent generative factors \cite{bengio2014representationlearningreviewnew, wang2024disentangledrepresentationlearning}. This decomposition is foundational for interpretability in scientific domains--such as genomics and physics--where latent dimensions must capture specific physical or biological mechanisms \cite{osti_1875308, cropsal2025compressingbiologyevaluatingstable}. Variational autoencoders (VAEs) provide the standard framework for this task \cite{kingma2013vae}, with the $\beta$-VAE \cite{higgins2017betavae} remaining a canonical baseline. By upweighting the Kullback--Leibler (KL) divergence in the evidence lower bound (ELBO), the $\beta$-VAE enforces the latent factorization essential for unsupervised disentanglement.

However, a persistent empirical failure mode limits this approach: \textit{high-$\beta$ collapse}. In practice, disentanglement performance--quantified by metrics such as the information-theoretic MIG \cite{chen2018isolating} and the predictor-based SAP \cite{kumar2018dipvae}--often peaks prematurely and deteriorates as regularization increases. While previous studies established that latent inference quality is non-monotonic in $\beta$, this degradation is usually loosely attributed to excessive regularization or a generic loss of reconstruction fidelity \cite{fil2021betavaereproducibilitychallengesextensions, ichikawa2024learning, ichikawa2025high}. A formal mechanistic account of why these specific evaluation scores mathematically vanish--even as the objective continues to prioritize latent independence--remains elusive.

In this paper, we identify a fundamental principle: the validity of these metrics is functionally dependent on the informativeness of the latent channel. In the linear-Gaussian regime, both MIG and SAP serve as estimators of the shared information between latents and generative factors. We argue that if the latent channel carries vanishing mutual information, these metrics become inherently \textit{degenerate}, losing the ability to distinguish between a factorized representation and an uninformative one. We characterize high-$\beta$ degradation as an \textit{informational collapse} phenomenon, where extreme KL pressure decouples the representation from the semantic information it was designed to organize.

To formalize this, we bridge recent high-dimensional asymptotic analyses of VAE dynamics \cite{ichikawa2024learning, ichikawa2025high} with a tractable linear-Gaussian framework. We demonstrate that for $\beta > 1$, the stationary points of the $\beta$-VAE objective trigger a \textit{spectral contraction} of the encoder gain, driving latent mutual information toward zero \cite{spectral_contration}. Motivated by this diagnosis, we analyze the $\lambda\beta$-VAE \cite{vu2025phd}, an objective that utilizes an auxiliary $L_2$ reconstruction penalty to decouple disentangling pressure from informational loss. Unlike non-stationary heuristics, the $\lambda$ parameter modifies the encoder stationarity conditions to explicitly counteract spectral contraction, significantly widening the stable operating regime for informative representations.

\subsection*{Summary of Contributions}

\begin{itemize}[leftmargin=*]
    \item \textbf{Analytical Diagnosis of Metric Degeneracy:} We prove that MIG and SAP share an implicit structural dependence on latent informativeness. We demonstrate that as information collapses, these metrics become mathematically degenerate, explaining the ``failed'' disentanglement scores observed in high-regularization regimes \cite{fil2021betavaereproducibilitychallengesextensions, carbonneau2022measuringdisentanglementreviewmetrics}.

    \item \textbf{Mechanistic Account of Informational Collapse:} Within a linear-Gaussian framework, we prove that for $\beta > 1$, stationarity-induced dynamics trigger a spectral contraction of the encoder gain. This establishes a \textit{low-information equilibrium} where the latent representation becomes asymptotically uninformative about generative factors, providing a rigorous basis for the phase-transition dynamics of informational collapse.
    
    \item \textbf{Principled Remedy via the $\lambda\beta$-VAE:} We characterize the $\lambda\beta$-VAE objective as a mechanism to decouple disentangling pressure from informational loss. We show that the $\lambda$ parameter introduces a structural intervention that modifies stationarity conditions to counteract spectral contraction, preserving latent signal under high regularization.
    
    \item \textbf{Empirical Validation:} We validate our framework on dSprites \cite{dsprites17}, Shapes3D \cite{3dshapes18}, and MPI3D-real \cite{gondal2019transferinductivebiassimulation}. Results confirm that $\lambda > 0$ stabilizes disentanglement metrics and reconstruction fidelity over a substantially broader operating regime than the standard $\beta$-VAE.
\end{itemize}

\section{Related Work}
\label{sec:related_work}

\paragraph{VAE-based Disentanglement.}
The $\beta$-VAE \cite{higgins2017betavae} established the principle that upweighting the KL divergence promotes latent factorization by constraining informational capacity. Refinements such as FactorVAE \cite{kim2018factorvae} and $\beta$-TCVAE \cite{chen2018isolating} specifically target total correlation, while DIP-VAE \cite{kumar2018dipvae} uses moment-matching to encourage an isotropic aggregated posterior. These methods operate on the philosophy that disentanglement is a property of increased regularization pressure. However, this narrative often overlooks the trade-off where the pressure intended to factorize the latent space simultaneously destroys the informativeness required for evaluation.

\paragraph{Theoretical Limits and Phase Transitions.}
Unsupervised disentanglement is fundamentally non-identifiable without explicit inductive biases or specific data assumptions \cite{locatello2019challengingcommonassumptionsunsupervised}. This structural fragility is amplified in the $\beta$-VAE, where increasing regularization beyond $\beta > 1$ forces a prioritization of latent independence that often necessitates the sacrifice of reconstruction fidelity \cite{mathieu2019disentangling}. Recent high-dimensional asymptotic analyses by \citet{ichikawa2024learning} and \citet{ichikawa2025high} characterize this behavior as an inevitable phase transition into posterior collapse once the KL weight exceeds a critical threshold. While \citet{wang2022posterior} attributed this collapse to a competition between likelihood and prior regularization in linear models, the literature still lacks a mechanistic account of how these stationary points trigger the mathematical vanishing of the specific informativeness-dependent scores used to evaluate disentanglement.

\paragraph{Metric Informativeness and Degeneracy.}
Evaluation typically relies on metrics quantifying statistical coupling, such as MIG \cite{chen2018isolating}, SAP \cite{kumar2018dipvae}, and DCI \cite{eastwood2018dci}. These metrics share a structural dependence on the informativeness of the latent channel \cite{abdi2019preliminarystudydisentanglementinsights, carbonneau2022measuringdisentanglementreviewmetrics}. While previous work viewed this through information bottleneck theory \cite{burgess2018understanding}, the explicit link between the spectral contraction of the encoder gain and the resulting degeneracy of these scoring mechanisms is largely unexplored. We argue that as information collapses, these metrics become mathematically degenerate, losing the ability to distinguish factorization from uninformative noise.

\paragraph{Information Recovery in Generative Backbones.}
VAEs serve as foundational backbones for modern hybrid architectures, most notably Latent Diffusion Models \cite{ho2020ddpm, rombach2022ldm}, where a highly informative latent space is essential for effective downstream controllability. Recent studies have highlighted that the heavy regularization required for latent factorization often strips the representations of critical factor-relevant information, compromising performance in high-stakes domains such as phenotypic drug discovery \cite{cropsal2025compressingbiologyevaluatingstable}. Consequently, investigating mechanisms that preserve informational capacity without sacrificing disentangling pressure has become a critical necessity for the stability and utility of modern representation learning backbones.

\section{Linear-Gaussian Framework}
\label{sec:linear_gaussian}

We utilize a linear-Gaussian framework to provide a mechanistic account of high-$\beta$ degradation. While linear-Gaussian VAEs are established diagnostic tools \cite{mathieu2019disentangling, ichikawa2025high}, we leverage their tractability to formalize the informational collapse hypothesized in Section~\ref{sec:introduction}.

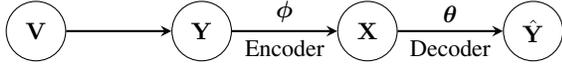
\begin{figure}[htb]
\centering
\begin{tikzpicture}[
  node distance=1.4cm,
  auto,
  every node/.style={scale=0.9},
  transform shape
]
  \node[state] (V) {$\mathbf{V}$};
  \node[state, right=of V] (Y) {$\mathbf{Y}$};
  \node[state, right=of Y] (X) {$\mathbf{X}$};
  \node[state, right=of X] (Yhat) {$\hat{\mathbf{Y}}$};

  \draw[arrow] (V) -- (Y);
  \draw[arrow] (Y) -- node[above] {$\boldsymbol{\phi}$} node[below] {Encoder} (X);
  \draw[arrow] (X) -- node[above] {$\boldsymbol{\theta}$} node[below] {Decoder} (Yhat);
\end{tikzpicture}
\caption{Linear-Gaussian VAE: Ground-truth factors $\mathbf{V}$ generate observations $\mathbf{Y}$; the encoder maps $\mathbf{Y}\mapsto \mathbf{X}$; the decoder yields reconstructions $\hat{\mathbf{Y}}$.}
\label{markov_chain}
\end{figure}

\subsection{Generative Model and Inference Structure}

The generative process is modeled as the Markov chain $\mathbf{V} \to \mathbf{Y} \to \mathbf{X} \to \hat{\mathbf{Y}}$ (Fig.~\ref{markov_chain}). Let $\mathbf{V} \in \mathbb{R}^s$ denote the ground-truth factors distributed as $\mathbf{V} \sim \mathcal{N}(\mathbf{0}, \mathbf{\Sigma}_{\mathbf{V}})$, where $\mathbf{\Sigma}_{\mathbf{V}}$ is diagonal and positive definite ($\mathbf{\Sigma}_{\mathbf{V}} \succ \mathbf{0}$). Observations $\mathbf{Y} \in \mathbb{R}^n$ are generated via a mixing matrix $\mathbf{\Gamma} \in \mathbb{R}^{n \times s}$:
\begin{equation}
\mathbf{Y} = \mathbf{\Gamma}\mathbf{V} + \tilde{\mathbf{Z}}, \qquad \tilde{\mathbf{Z}} \sim \mathcal{N}(\mathbf{0}, \sigma^2\mathbf{I}_n),
\end{equation}
with observation covariance $\mathbf{\Sigma}_{\mathbf{Y}} = \mathbf{\Gamma}\mathbf{\Sigma}_{\mathbf{V}}\mathbf{\Gamma}^\top + \sigma^2\mathbf{I}_n$. 

The latent representation $\mathbf{X} \in \mathbb{R}^m$ follows an isotropic prior $p(\mathbf{x}) = \mathcal{N}(\mathbf{0}, \mathbf{I}_m)$. The affine Gaussian encoder $q_{\boldsymbol{\phi}}(\mathbf{x}|\mathbf{y})$ employs gain $\mathbf{B} \in \mathbb{R}^{m \times n}$ and noise $\mathbf{W} \sim \mathcal{N}(\mathbf{0}, \mathbf{\Sigma}_{\mathbf{W}})$:
\begin{equation}
\mathbf{X} = \mathbf{B}\mathbf{Y} + \mathbf{W}, \qquad \mathbf{\Sigma}_{\mathbf{W}} \succ \mathbf{0}.
\end{equation}
The decoder reconstructs $\hat{\mathbf{Y}}$ via gain $\mathbf{A} \in \mathbb{R}^{n \times m}$ and noise $\mathbf{Z} \sim \mathcal{N}(\mathbf{0}, \mathbf{\Sigma}_{\mathbf{Z}})$:
\begin{equation}
\hat{\mathbf{Y}} = \mathbf{A}\mathbf{X} + \mathbf{Z}, \qquad \mathbf{\Sigma}_{\mathbf{Z}} \succ \mathbf{0}.   
\end{equation}
Crucially, the statistical coupling between latents and generative factors is fully characterized by the joint covariance of $(\mathbf{X}, \mathbf{V})$:
\begin{equation}
\label{cov_xv}
\mathbf{\Sigma}_{(\mathbf{X},\mathbf{V})} = 
\begin{bmatrix}
\mathbf{\Sigma}_{\mathbf{X}} & \mathbf{B}\mathbf{\Gamma}\mathbf{\Sigma}_{\mathbf{V}} \\
\mathbf{\Sigma}_{\mathbf{V}}\mathbf{\Gamma}^\top\mathbf{B}^\top & \mathbf{\Sigma}_{\mathbf{V}}
\end{bmatrix},
\end{equation}
where $\mathbf{\Sigma}_{\mathbf{X}} = \mathbf{B}\mathbf{\Sigma}_{\mathbf{Y}}\mathbf{B}^\top + \mathbf{\Sigma}_{\mathbf{W}}$. This structure determines the mutual information $I(\mathbf{X};\mathbf{V})$. As $\beta$ increases, we track the spectral contraction of $\mathbf{B}$, which drives the cross-covariance $\mathbf{B}\mathbf{\Gamma}\mathbf{\Sigma}_{\mathbf{V}}$ to zero.

\subsection{Variational Objectives and the $\lambda\beta$-VAE Extension}

The $\beta$-VAE objective is the weighted ELBO:
\begin{equation}
\mathcal{L}_{\beta} = \underbrace{-\mathbb{E}_{q_{\boldsymbol{\phi}}(\mathbf{x}|\mathbf{y})}[\log p_{\boldsymbol{\theta}}(\mathbf{y}|\mathbf{x})]}_{\text{Reconstruction Error}} + \beta \underbrace{\mathrm{KL}(q_{\boldsymbol{\phi}}(\mathbf{x}|\mathbf{y}) \| p(\mathbf{x}))}_{\text{Regularization Pressure}}.
\end{equation}
In Section~\ref{sec:collapse}, we prove that for $\beta > 1$, the KL-induced stationarity dynamics force $\mathbf{B} \to \mathbf{0}$, causing $I(\mathbf{X};\mathbf{V}) \to 0$. To resolve this, we analyze the $\lambda\beta$-VAE \cite{vu2025phd}, which introduces an auxiliary reconstruction-pressure term:
\begin{equation}
\mathcal{L}_{\lambda\beta} = \mathcal{L}_{\beta} + \lambda\,\mathbb{E}_{ q_{\boldsymbol{\phi}}(\mathbf{x},\mathbf{y})}[\|\mathbf{Y} - \hat{\mathbf{Y}}\|^2], \qquad \lambda \ge 0.
\end{equation}
By decoupling semantic reconstruction from the KL bottleneck, $\lambda$ modifies the stationarity conditions to counteract contraction, preserving the integrity of $\mathbf{\Sigma}_{(\mathbf{X},\mathbf{V})}$ even under high-regularization regimes.

\subsection{Informational Collapse in the $\beta$-VAE}
\label{sec:collapse}

We characterize informational collapse as the asymptotic vanishing of the latent signal at the objective’s stationary points. In the linear-Gaussian regime, these points are analytically governed by a system of coupled stationarity conditions \cite{vu2025phd}, as established in the following lemma:

\begin{lemma}[$\beta$-VAE Stationary Point]
\label{lem:beta_opt}
At any stationary point of the $\beta$-VAE objective, the model parameters $(\mathbf{A}, \mathbf{B}, \mathbf{\Sigma}_{\mathbf{Z}}, \mathbf{\Sigma}_{\mathbf{W}})$ must satisfy the following system of conditions:
\begin{align*}
\mathbf{A} &= (\mathbf{\Sigma}_{\mathbf{Y}}^{-1} + \mathbf{B}^\top\mathbf{\Sigma}_{\mathbf{W}}^{-1}\mathbf{B})^{-1} \mathbf{B}^\top\mathbf{\Sigma}_{\mathbf{W}}^{-1}, \\
\mathbf{B} &= \bigl[\mathbf{I}_m + \mathbf{A}^\top(\mathbf{\Sigma}_{\mathbf{Z}}^{-1}/\beta)\mathbf{A}\bigr]^{-1} \mathbf{A}^\top(\mathbf{\Sigma}_{\mathbf{Z}}^{-1}/\beta), \\
\mathbf{\Sigma}_{\mathbf{Z}} &= (\mathbf{\Sigma}_{\mathbf{Y}}^{-1} + \mathbf{B}^\top\mathbf{\Sigma}_{\mathbf{W}}^{-1}\mathbf{B})^{-1}, \\
\mathbf{\Sigma}_{\mathbf{W}} &= \bigl[\mathbf{I}_m + \mathbf{A}^\top(\mathbf{\Sigma}_{\mathbf{Z}}^{-1}/\beta)\mathbf{A}\bigr]^{-1}.
\end{align*}
\end{lemma}

The explicit $\beta^{-1}$ scaling within the encoder gain update for $\mathbf{B}$ serves as the fundamental mechanism for informational suppression.

\begin{theorem}[Informational Collapse for $\beta>1$]
\label{thm:trivial}
For any $\beta>1$, the stationarity-induced dynamics characterized by Lemma~\ref{lem:beta_opt} converge to the trivial solution $(\mathbf{A},\mathbf{B},\mathbf{\Sigma}_{\mathbf{Z}},\mathbf{\Sigma}_{\mathbf{W}}) = (\mathbf{0},\mathbf{0},\mathbf{\Sigma}_{\mathbf{Y}}, \mathbf{I}_m)$. Consequently, the latent representation becomes asymptotically uninformative:
\begin{equation*}
I(\mathbf{X};\mathbf{V}) \;\xrightarrow[\beta > 1]{} \; 0.
\end{equation*}
\end{theorem}

\textit{Proof Sketch.} The proof analyzes the spectral dynamics of the coupled stationarity updates derived in Lemma~\ref{lem:beta_opt}. We first prove that the encoder noise covariance is uniformly bounded in the Loewner order as $\mathbf{\Sigma}_{\mathbf{W}} \preceq \mathbf{I}_m$, which implies the spectral bound $\|\mathbf{\Sigma}_{\mathbf{W}}\|_2 \le 1$. This bound is a necessary condition to establish the contraction of the encoder gain; by iterating the stationarity conditions, we derive a recursive identity for $\mathbf{B}$ where the regularization weight enters as an explicit $\beta^{-n}$ contraction factor. For $\beta > 1$, this induces an exponential spectral contraction driving $\|\mathbf{B}\|_2 \to 0$. This forces the cross-covariance $\mathbf{B}\mathbf{\Gamma}\mathbf{\Sigma}_{\mathbf{V}}$ in the joint covariance $\mathbf{\Sigma}_{(\mathbf{X},\mathbf{V})}$ to vanish, resulting in the informational collapse $I(\mathbf{X};\mathbf{V}) \to 0$. See Appendix~\ref{app:collapse_proof} for the full derivation.

Theorem~\ref{thm:trivial} establishes the mechanistic basis for the failed disentanglement scores observed in high-regularization regimes. As we demonstrate in Section~\ref{sec:metrics}, metrics such as MIG and SAP share a structural dependence on $I(\mathbf{X};\mathbf{V})$, rendering them mathematically degenerate upon spectral collapse.

\section{Metric Degeneracy and Restoration}
\label{sec:metrics}

\subsection{Analytical Metric Formulations}

We analyze three representative metrics: Separated Attribute Predictability (SAP), Mutual Information Gap (MIG), and the latent informativeness metric $I_m$ \cite{vu2025phd}. While MIG is standard in non-linear settings, we introduce $I_m$ as a principled alternative for continuous variables in the linear-Gaussian regime. We demonstrate that all three metrics share a structural dependence on the encoder gain, rendering them degenerate as the latent representation becomes asymptotically uninformative.

\paragraph{Separated Attribute Predictability (SAP).}
SAP \cite{kumar2018dipvae} is governed by the squared correlation matrix $\mathbf{S} \in \mathbb{R}^{m \times s}$, where elements $S_{i,j}$ represent the predictability of factor $V_j$ from latent $X_i$:
\begin{equation}
S_{i,j} = \frac{(\mathbf{B}\mathbf{\Gamma}\mathbf{\Sigma}_{\mathbf{V}})_{i,j}^2}{(\mathbf{B}\mathbf{\Sigma}_{\mathbf{Y}}\mathbf{B}^\top + \mathbf{\Sigma}_{\mathbf{W}})_{i,i} (\mathbf{\Sigma}_{\mathbf{V}})_{j,j}}.
\end{equation}
The SAP score is the normalized difference between the top two entries in each column:
\begin{equation}\label{SAP}
    \text{SAP}(\mathbf{V},\mathbf{X}) = \frac{1}{s}\sum_{j=1}^s\left(S_{i^{(j)},j}-S_{i'^{(j)},j}\right),
\end{equation}
where $i^{(j)}$ and $i'^{(j)}$ are the indices of the most and second-most predictive latent variables for $V_j$. As the encoder gain $\mathbf{B} \to \mathbf{0}$, the squared correlation $S_{i,j}$ vanishes for all pairs, driving the gap toward a non-informative zero-limit.

\paragraph{Mutual Information Gap (MIG).}
MIG \cite{chen2018isolating} measures the normalized gap in mutual information between the top two latent candidates for a given factor:
\begin{equation}\label{MIG_score} 
\text{MIG}(\mathbf{V}, \mathbf{X})  = \frac{1}{s} \sum_{j=1}^s \frac{I(X_{i^{(j)}}; V_j) - I(X_{i'^{(j)}}; V_j)}{H(V_j)}. 
\end{equation}
In our Gaussian setting, $I(X_i; V_j)$ is a monotonic function of $S_{i,j}$, specifically $I(X_i; V_j) = -\frac{1}{2} \log ( 1 - S_{i,j} )$. Consequently, MIG becomes mathematically degenerate as the spectral norm $\|\mathbf{B}\|_2 \to 0$ forces $S_{i,j}$ to zero.

\paragraph{The $I_m$ Metric.}
To quantify latent informativeness, we utilize the $I_m$ metric \cite{vu2025phd}, which rewards the alignment of latent dimensions to generative factors while penalizing redundant representations. Given independent factors, the redundancy terms vanish and the metric reduces to a maximum-weight bipartite matching problem \cite{doi:10.1137/1.9781611973099.111} that maximizes the total mutual information between latents and factors:
\begin{equation}\label{piecewise_m}
I_m(\mathbf{V}, \mathbf{X}) = \max_{\mathcal{P}} \sum_{k=1}^m I(X_k; \mathcal{V}_{s_k}),
\end{equation}
where $\mathcal{P} = \{\mathcal{V}_{s_k}\}_{k=1}^m$ denotes a partition of the $s$ generative factors into $m$ disjoint subsets. In our linear-Gaussian framework, $I_m$ provides a stable evaluation of the latent channel capacity; unlike MIG, which relies on entropy-based normalization that can become numerically unstable for continuous variables, $I_m$ directly tracks the total preserved signal across the latent bottleneck.

\subsection{Asymptotic Degeneracy and $\lambda$-Restoration}

By mapping the stationarity conditions derived in Section~\ref{sec:collapse} to metric definitions, we formalize the failure of informativeness-dependent evaluation:

\begin{corollary}[Metric Degeneracy for $\beta>1$]
\label{cor:collapse}
In the linear-Gaussian $\beta$-VAE, the stationarity-induced dynamics for $\beta > 1$ necessitate an exponential spectral contraction of the encoder gain $\|\mathbf{B}\|_2 \to 0$. This contraction forces the pairwise correlations $S_{i,j}$ and latent mutual information $I(X_i; V_j)$ to zero for all $(i,j)$, rendering informativeness-dependent metrics mathematically degenerate:
\begin{equation*}
    \{\mathrm{SAP}, \mathrm{MIG}, I_m\} \xrightarrow[\beta > 1]{} 0.
\end{equation*}
\end{corollary}

This identifies high-$\beta$ degradation not as a failure of disentangling pressure, but as a total erasure of the signal required for measurement. To counteract this collapse, the $\lambda\beta$-VAE objective introduces a structural modification to the underlying stationarity conditions:

\begin{lemma}[$\lambda\beta$-VAE Stationary Point]
\label{lem:lbeta_opt}
Let $\mathbf{M} \coloneqq (\mathbf{\Sigma}_{\mathbf{Z}}^{-1} + 2\lambda\mathbf{I}_n)/\beta$. At any stationary point of the $\lambda\beta$-VAE objective, the decoder parameters satisfy the conditions established in Lemma~\ref{lem:beta_opt}, while the encoder parameters satisfy:
\begin{align*}
    \mathbf{B} &= (\mathbf{I}_m + \mathbf{A}^\top \mathbf{M} \mathbf{A})^{-1} \mathbf{A}^\top \mathbf{M}, \\
    \mathbf{\Sigma}_{\mathbf{W}} &= (\mathbf{I}_m + \mathbf{A}^\top \mathbf{M} \mathbf{A})^{-1}.
\end{align*}
\end{lemma}

The $\lambda$ parameter acts as a stationarity intervention that modulates the spectral decay induced by $\beta$. Specifically, the $2\lambda\mathbf{I}_n$ term serves as a \textit{damping term} within $\mathbf{M}$ that counteracts the vanishing of the encoder gain. By decoupling the regularization of latent variance from the reconstruction pressure, $\lambda$ shifts the informational phase transition, effectively widening the stable regime where information about generative factors remains measurable and preserved.

\begin{figure}[htb]
\centering
\includegraphics[width=\columnwidth]{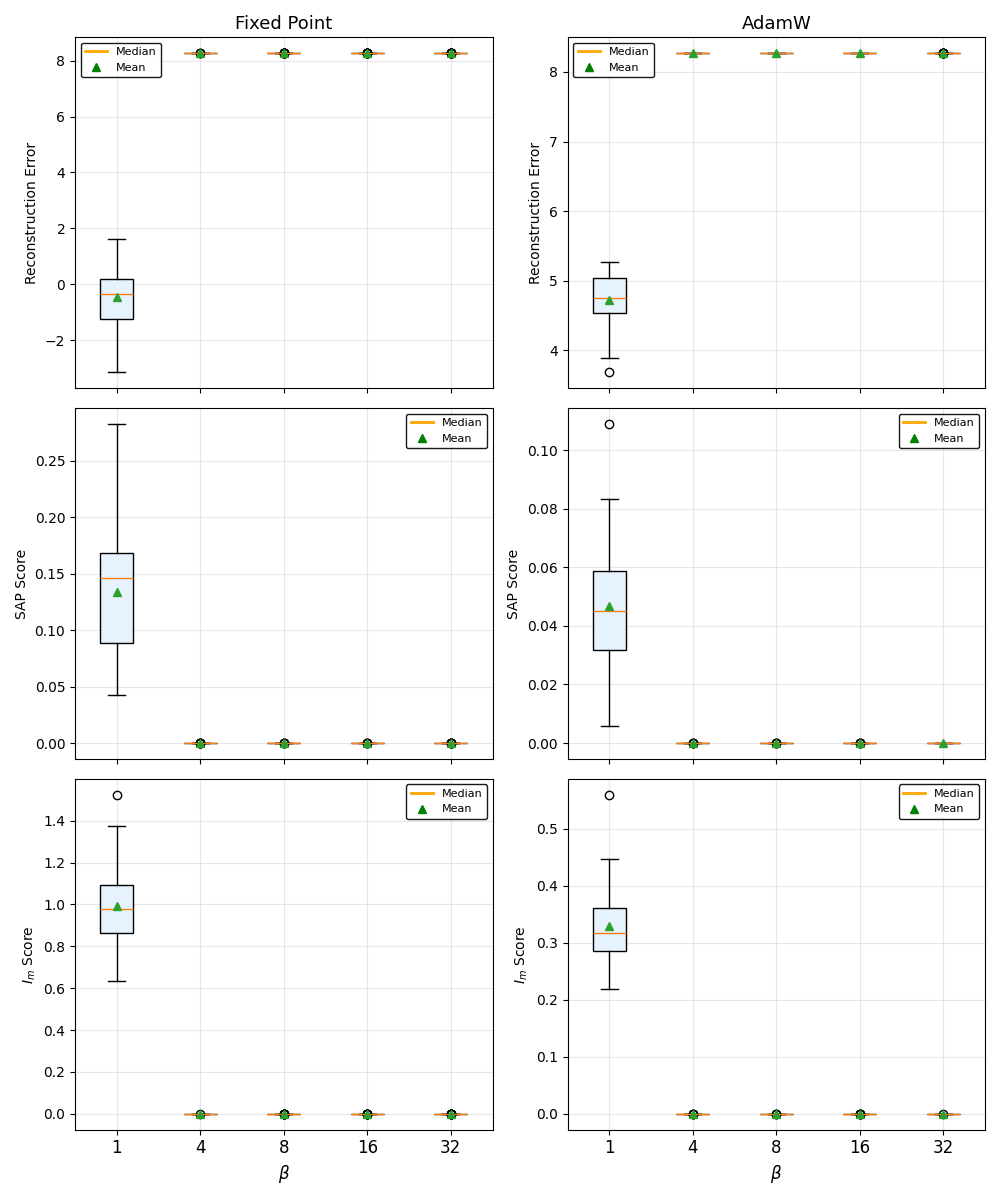}
\caption{Reconstruction error, SAP score, and
$I_m$ score for the $\beta$-VAE with $(n,m,s)=(100,10,5)$, using
fixed-point iteration and AdamW. For $\beta>1$, SAP and $I_m$ collapse toward
zero across both optimization procedures, consistent with convergence to the
trivial low-information solution in Theorem~\ref{thm:trivial}.}
\label{fig:lbeta_metrics}
\end{figure}

\begin{figure}[htb]
\centering
\includegraphics[width=\columnwidth]{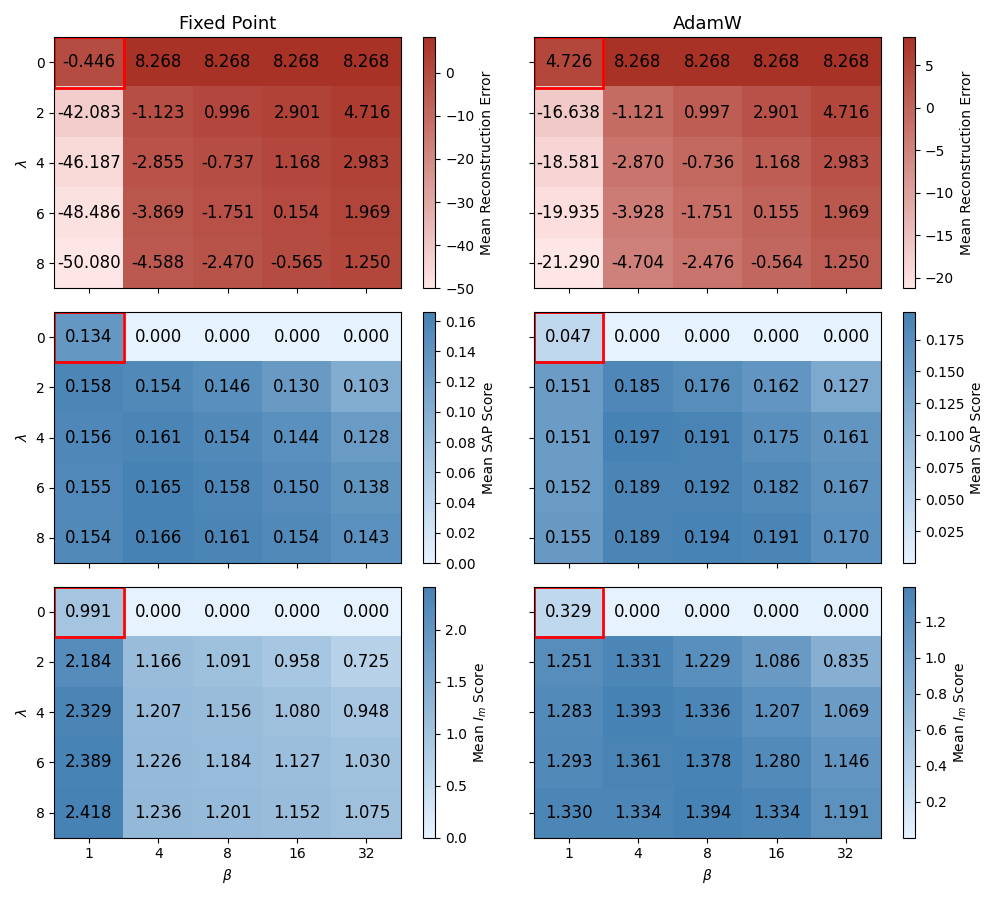}
\caption{Mean reconstruction error, SAP score,
and $I_m$ score for the $\lambda\beta$-VAE with $(n,m,s)=(100,10,5)$. Positive
$\lambda$ prevents the collapse observed at $\lambda=0$ and preserves informative
latent-factor dependence at larger $\beta$, yielding non-degenerate SAP and
$I_m$.}
\label{fig:llambdabeta_metrics}
\end{figure}

\begin{table}[htb]
\centering
\small
\caption{VAE architecture used in nonlinear experiments.}
\label{tab:vae_architecture}
\begin{tabular}{@{}l@{}}
\toprule
\textbf{Encoder} \\
\midrule
Input: $64 \times 64 \times C$ \\
Conv $(4\times4,32)$, stride 2, ReLU \\
Conv $(4\times4,32)$, stride 2, ReLU \\
Conv $(4\times4,64)$, stride 2, ReLU \\
Conv $(4\times4,64)$, stride 2, ReLU \\
FC 256, ReLU \\
FC $2m$ (mean and log-variance) \\
\midrule
\textbf{Decoder} \\
\midrule
Input: $\mathbb{R}^{m}$ \\
FC 256, ReLU \\
FC $4 \times 4 \times 64$, ReLU \\
ConvT $(4\times4,64)$, stride 2, ReLU \\
ConvT $(4\times4,32)$, stride 2, ReLU \\
ConvT $(4\times4,32)$, stride 2, ReLU \\
ConvT $(4\times4,C)$, stride 2\\
\bottomrule
\end{tabular}
\end{table}

\begin{table}[htb]
\centering
\small
\caption{Training hyperparameters for nonlinear experiments.}
\label{tab:training_params}
\begin{tabular}{@{}ll@{}}
\toprule
\textbf{Parameter} & \textbf{Value} \\
\midrule
Optimizer & Adam \\
Learning rate & $10^{-4}$ \\
Batch size & 64 \\
$\beta$ grid & $\{1,4,8,16,32,64,128,256\}$ \\
$\lambda$ grid & $\{0,4,8,16,32\}$ \\
$\beta$-VAE training & $150{,}000$ steps \\
$\lambda\beta$-VAE finetune & $50{,}000$ steps \\
Seeds & 10 \\
Latent dimension & $m=15$ \\
\bottomrule
\end{tabular}
\end{table}

\begin{figure*}[htb]
\centering
\includegraphics[width=0.8\textwidth]{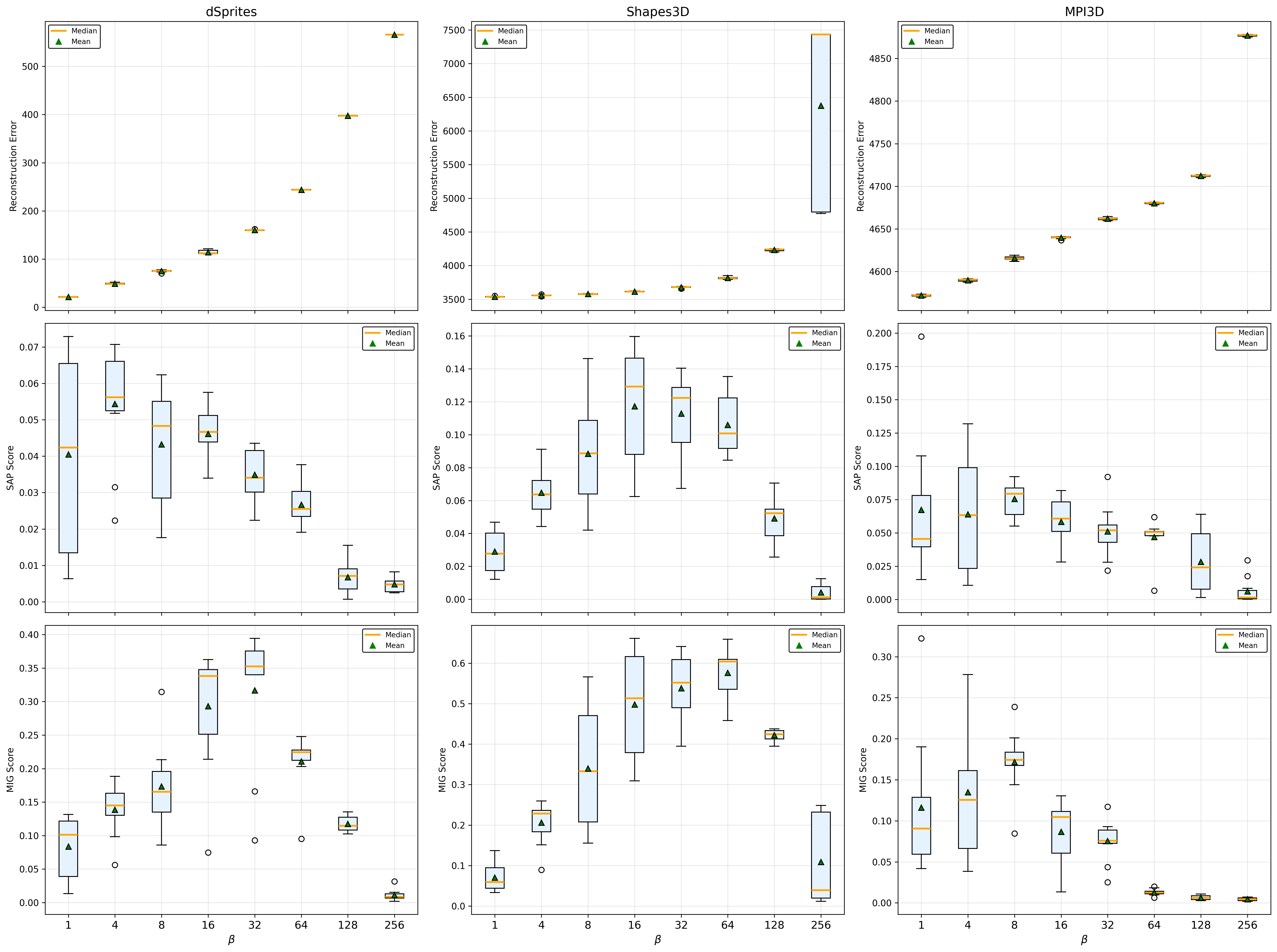}
\caption{Reconstruction error, SAP score, and MIG score for the $\beta$-VAE
across dSprites, Shapes3D, and MPI3D-real.
Disentanglement peaks at intermediate $\beta$ values and deteriorates at large
$\beta$.}
\label{fig:boxplots_across_datasets}
\end{figure*}

\begin{figure*}[htb]
\centering
\includegraphics[width=0.9\textwidth]{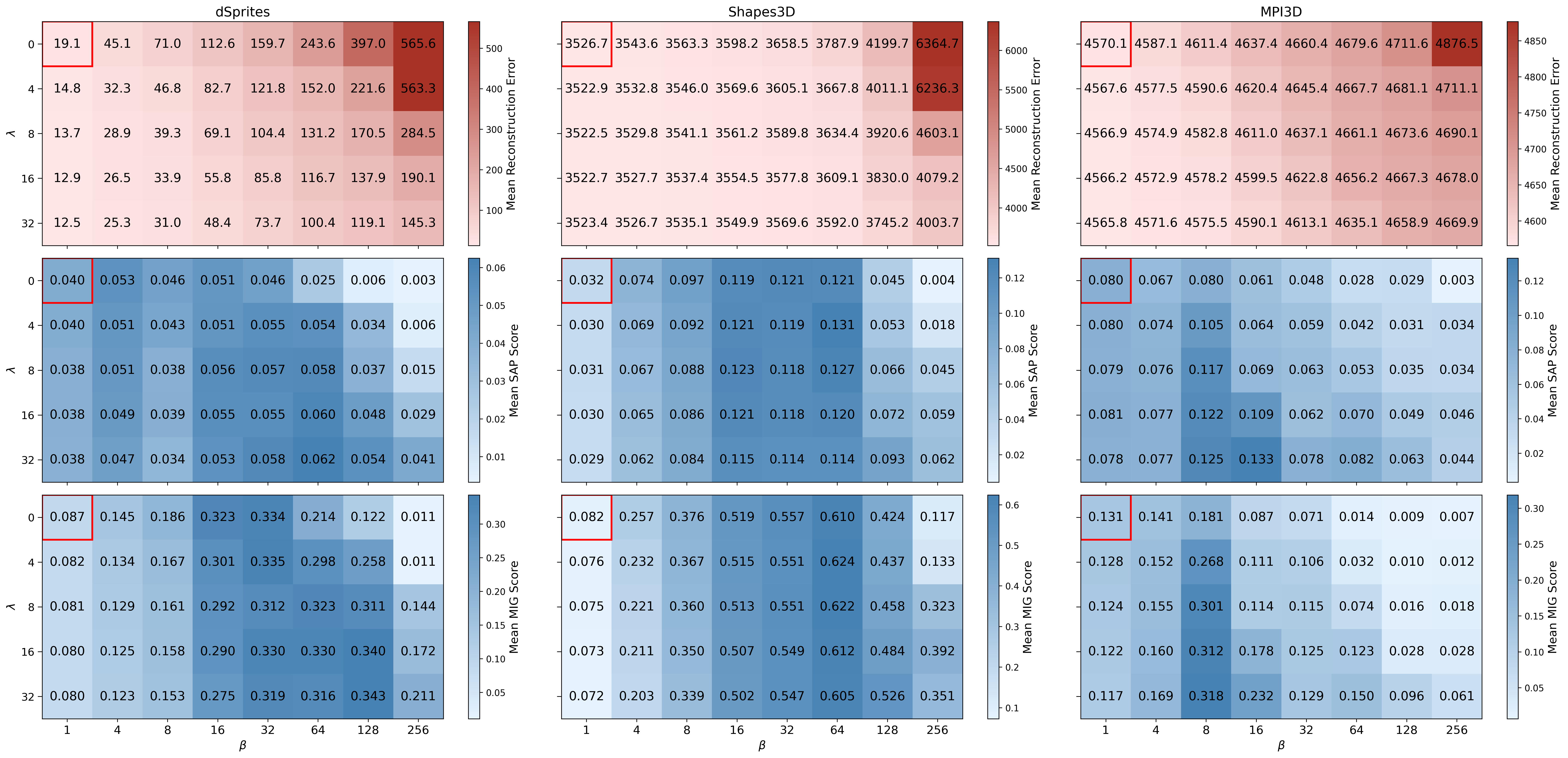}
\caption{Mean reconstruction error, SAP, and MIG for the $\lambda\beta$-VAE across
dSprites, Shapes3D, and MPI3D-real.
Positive $\lambda$ improves reconstruction and preserves disentanglement at
large $\beta$, consistent with information preservation.}
\label{fig:heatmaps_across_datasets}
\end{figure*}

\section{Empirical Evaluation}
\label{sec:experiments}

We evaluate the $\lambda\beta$-VAE framework through two complementary studies: (i) controlled linear-Gaussian simulations to verify the stationarity-induced phase transitions derived in Section~\ref{sec:collapse}, and (ii) large-scale visual benchmarks to test the robustness of informational restoration in deep convolutional architectures.

\subsection{Mechanistic Verification: Linear-Gaussian Regime}
\label{sec:numerics_linear}

We first validate the theoretical existence of the trivial solution for $\beta > 1$ using a high-dimensional configuration ($n=100$, $m=10$, $s=5$). Factors $\mathbf{V}$ are sampled i.i.d.\ from $\mathcal{U}(0.1, 1.0)$ with observation noise $\sigma^2 = 0.05$. To isolate the objective's properties from optimizer bias, we compare two distinct dynamics: \textbf{Analytical Fixed-Point Iteration} of the updates in Lemma~\ref{lem:beta_opt} and \textbf{Stochastic Optimization} (AdamW) with Cholesky-parameterized covariances. For each $(\beta, \lambda)$ pair, we conduct 50 independent trials with random initialization to ensure statistical rigor.\footnote{Code repository: \href{https://github.com/mh-vu/lambda-beta-vae}{https://github.com/mh-vu/lambda-beta-vae}}

\paragraph{Results.} As predicted by Theorem~\ref{thm:trivial}, the baseline $\beta$-VAE ($\lambda=0$) exhibits an abrupt phase transition (Fig.~\ref{fig:lbeta_metrics}). For $\beta > 1$, both analytical and stochastic trajectories reliably converge to the trivial solution, where the encoder gain $\mathbf{B}$ contracts to zero and the SAP and $I_m$ scores vanish. Conversely, introducing $\lambda > 0$ fundamentally alters the stationarity conditions (Lemma~\ref{lem:lbeta_opt}), counteracting this spectral contraction (Fig.~\ref{fig:llambdabeta_metrics}). Under the $\lambda\beta$-VAE objective, the latent-factor coupling $\mathbf{\Sigma}_{\mathbf{XV}}$ remains preserved even under high regularization pressure ($\beta=32$), validating $\lambda$ as a principled mechanism for widening the stable operating regime of informative representations.

\subsection{Diagnostic Validation: Deep Nonlinear Architectures}
\label{sec:nonlinear}

We extend our evaluation to a hierarchy of visual complexity: dSprites, Shapes3D, and MPI3D-real. The latter serves as a critical test for reality transfer, where realistic visual correlations and complex backgrounds typically accelerate informational decay in over-regularized regimes \cite{schott2022visualrepresentationlearningdoes}.

\paragraph{Architecture and Continuation Protocol.} 
We utilize a deep convolutional VAE with a 15-dimensional latent space, implementing a Gaussian encoder and a Bernoulli decoder. Detailed architectural specifications and training hyperparameters are provided in Tables~\ref{tab:vae_architecture} and~\ref{tab:training_params}. To evaluate the efficacy of the $\lambda$ intervention, we employ a two-phase training protocol:

\begin{enumerate}[leftmargin=*]
    \item \textbf{Stabilization Phase:} Models are first trained as standard $\beta$-VAEs ($\lambda=0$) for $150{,}000$ steps across the $\beta$-grid. This phase allows the latent space to stabilize and attempt factorization under standard regularization pressure before any structural intervention is applied.
    \item \textbf{Restoration Phase:} We introduce the $\lambda$ parameter and continue training for an additional $50{,}000$ steps using the $\lambda\beta$-VAE objective. 
\end{enumerate}
This sequential protocol prevents the auxiliary penalty from interfering with early-stage disentanglement dynamics while specifically testing whether $\lambda$ can recover semantic informativeness in representations that have succumbed to high-$\beta$ informational collapse.

\paragraph{Empirical Findings.} 
As shown in Fig.~\ref{fig:boxplots_across_datasets}, all datasets exhibit the predicted high-$\beta$ degradation. While the transition in deep architectures is more gradual than in the linear case, the deterioration of SAP and MIG scores confirms that extreme KL pressure erases factor-relevant information. This effect is most pronounced in MPI3D-real, consistent with recent findings that increased dataset complexity lowers the threshold for posterior collapse \cite{ichikawa2025high}.

Heatmap analysis (Fig.~\ref{fig:heatmaps_across_datasets}) confirms that $\lambda$ provides a robust restorative effect. For models approaching the collapse regime at high $\beta$, introducing $\lambda > 0$ significantly stabilizes disentanglement metrics and improves overall scores while simultaneously enhancing reconstruction fidelity. This confirms that $\lambda$ serves as a structural control for latent informativeness, preventing the ``oversmoothing'' typically observed in over-regularized generative backbones \cite{takida2022preventingoversmoothingvaegeneralized}.

\begin{figure*}[htb]
\centering
\includegraphics[width=0.9\textwidth]{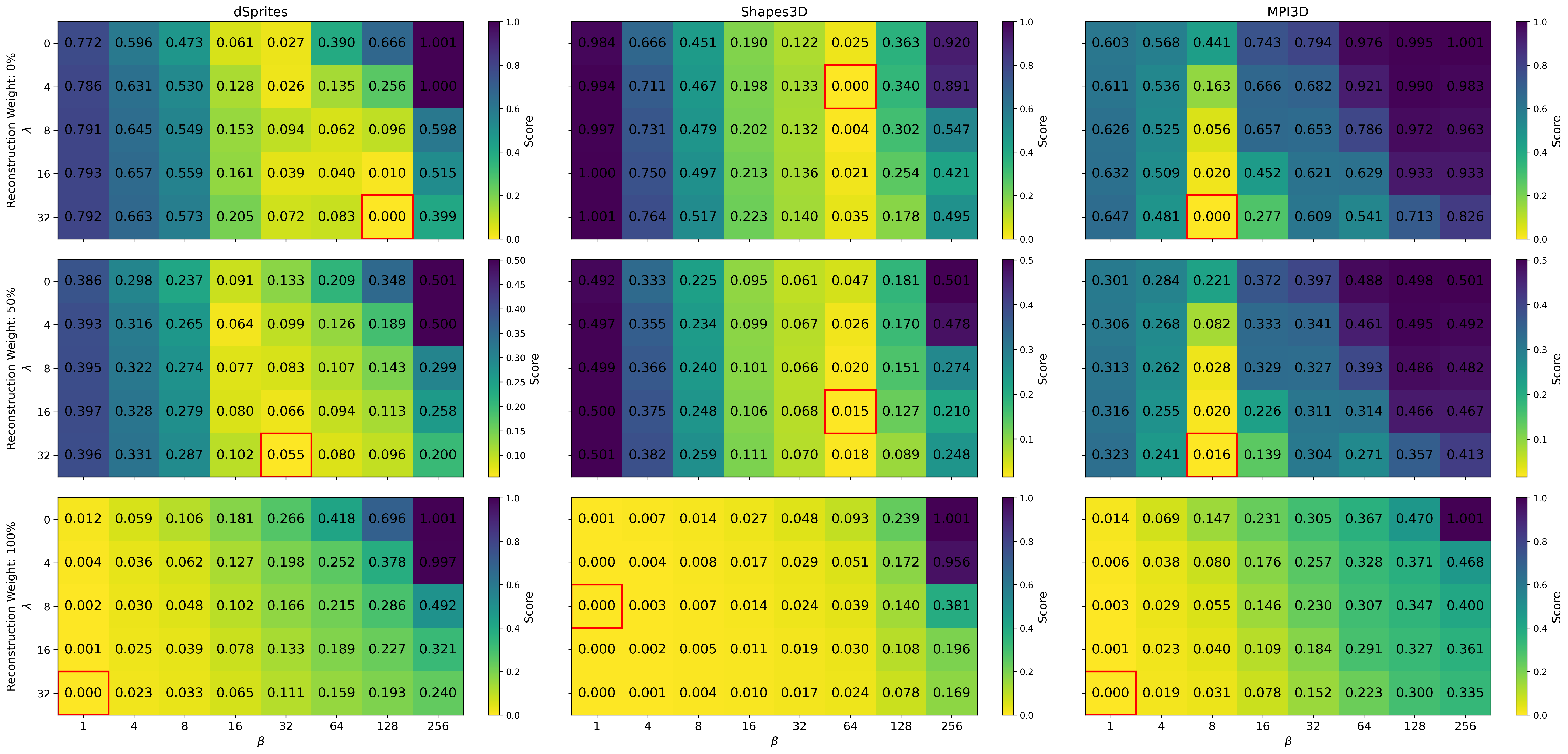}
\caption{Augmented Tchebycheff scalarization heatmaps for three preference
profiles. Red markers indicate the
minimizing $(\beta,\lambda)$ configurations over the evaluated grid.}
\label{fig:tchebycheff_heatmaps}
\end{figure*}

\section{Hyperparameter Selection Strategy}
\label{sec:hyperparam_selection}

Practical deployment of the $\lambda\beta$-VAE requires a principled method for selecting an optimal operating point within the hyperparameter space. We formalize this as a multi-objective optimization problem, identifying configurations $(\beta, \lambda) \in \mathcal{G}$ that satisfy the competing demands of reconstruction fidelity and representation quality. Specifically, we seek to minimize the objective vector $[f_1, f_2]^\top$, where $f_1$ represents the reconstruction error (NLL) and $f_2$ represents the \textit{entanglement score}, defined as $1 - \text{MIG}$.

\paragraph{Augmented Tchebycheff Scalarization.}
To navigate this trade-off, we employ the augmented Tchebycheff scalarization \cite{Steuer1983AnIW, miettinen1999nonlinear, chugh2019scalarizingfunctionsbayesianmultiobjective}. This approach is particularly suited for the non-convex Pareto fronts characteristic of VAE evaluation landscapes. We first map each objective $f_i$ to a normalized score $\bar{f}_i \in [0, 1]$ over the evaluation grid $\mathcal{G}$:
\begin{equation*}
    \bar{f}_i(\beta, \lambda) = \frac{f_i(\beta, \lambda) - \min\limits_{\mathcal{G}} f_i}{\max\limits_{\mathcal{G}} f_i - \min\limits_{\mathcal{G}} f_i}, \quad i \in \{1, 2\}.
\end{equation*}
Given preference weights $w_1, w_2 > 0$ with $\sum w_i = 1$, the scalarized selection score $\mathcal{S}$ is defined as:
\begin{equation*}
\mathcal{S}(\beta,\lambda; \mathbf{w}) = \max_{i \in \{1,2\}} \{w_i \bar{f}_i(\beta,\lambda)\} + \rho \sum_{i=1}^2 w_i \bar{f}_i(\beta,\lambda),
\end{equation*}
where $\rho = 10^{-3}$ is a small augmentation parameter ensuring strict Pareto optimality. By varying the weight vector $\mathbf{w}$, we can trace the Pareto front and select operating points that prioritize either data fidelity ($w_1$) or latent independence ($w_2$).

\paragraph{Interactive Selection and Trade-offs.}
Fig.~\ref{fig:tchebycheff_heatmaps} visualizes the selection scores across three representative preference profiles ($w_1 \in \{0, 0.5, 1\}$). As the priority shifts from reconstruction ($w_1=1$) to disentanglement ($w_1=0$), the optimal configuration migrates from low-$\beta$ to high-$\beta$ regimes. Critically, our numerical results demonstrate that for any fixed $\beta$ in the over-regularized regime, selecting a positive $\lambda$ recovers significant reconstruction fidelity without compromising the disentanglement score. This enables an interactive selection process where, by adjusting the weight vector $\mathbf{w}$, we can identify the specific $(\beta, \lambda)$ pair that best satisfies project-specific requirements for data fidelity and informative representation.

\section{Conclusion}
\label{sec:conclusion}

This paper identifies and resolves a foundational pathology in variational representation learning: the information-theoretic collapse of the latent channel under regularization. We have demonstrated that the non-monotonic performance observed in \(\beta\)-VAEs is a mathematically predictable consequence of spectral contraction. By formalizing this mechanism in a linear-Gaussian regime, we proved that for \(\beta > 1\), the stationarity-induced dynamics of the encoder drive latent-factor mutual information to zero. This diagnosis provides a rigorous explanation for the degeneracy of established metrics such as MIG and SAP, which, while intended to quantify factorization, fundamentally require a non-collapsed information substrate to remain valid.

The introduction of the \(\lambda\beta\)-VAE provides a principled remedy by decoupling the pressure for latent independence from the preservation of semantic signal. Our analytical proofs and cross-dataset validation on dSprites, Shapes3D, and MPI3D-real confirm that the auxiliary reconstruction weight \(\lambda\) acts as a spectral damping mechanism, counteracting informational collapse and significantly widening the stable operating regime for informative representations. While ensuring latent informativeness is a prerequisite for valid evaluation, we acknowledge that preserving the signal does not, by itself, resolve the fundamental non-identifiability of unsupervised disentanglement \cite{locatello2019challengingcommonassumptionsunsupervised}, as a non-collapsed latent channel may still lack the specific inductive biases required to align with ground-truth generative factors. Instead, it ensures the latent space maintains the statistical coupling necessary for any downstream inductive bias or causal discovery mechanism to act upon \cite{schölkopf2021causalrepresentationlearning, wang2024disentangledrepresentationlearning}.

Our results suggest several high-impact avenues for future research. First, the static hyperparameter selection utilized here could be evolved into adaptive \(\lambda\)-scheduling, where reconstruction pressure is dynamically adjusted via control-theory principles \cite{shao2020controlvae} to maintain a target mutual information. Second, the \(\lambda\beta\)-VAE objective should be investigated within hybrid generative backbones, such as Latent Diffusion Models \cite{rombach2022ldm}, where the informational integrity of the VAE-compressed latent space directly dictates the precision of downstream controllability. Finally, extending our spectral contraction analysis to non-Gaussian priors and discrete-continuous latent hybrids \cite{takida2022sqvaevariationalbayesdiscrete, cho2023hyperbolicvaelatentgaussian, pynadath2025candihybriddiscretecontinuousdiffusion} will be essential for developing a unified theory of informativeness across diverse data distributions.

By providing a mechanistic account of collapse and a robust tool for its prevention, this work establishes a clear design principle: \(\beta\) should modulate latent factorization, while a decoupled mechanism such as \(\lambda\) must safeguard informational integrity. This two-parameter control surface is a critical requirement for developing stable, interpretable, and valid representation learning models for complex scientific data structures and generative tasks alike.

\bibliography{paper}
\bibliographystyle{icml2025}

\newpage
\appendix
\onecolumn
\section{Proof of Theorem~\ref{thm:trivial}}
\label{app:collapse_proof}

This appendix provides the formal derivation of the informational collapse result for the $\beta$-VAE in a linear-Gaussian regime. Throughout, $\|\cdot\|_2$ denotes the spectral norm (the largest singular value).

\subsection{Uniform Spectral Bound on Encoder Covariance}

We establish a uniform bound on the encoder noise covariance in the Loewner order. This result is a necessary prerequisite to obtain $\|\mathbf{\Sigma}_{\mathbf{W}}\|_2 \le 1$, ensuring the convergence of the recursive gain dynamics in the proof of Theorem~\ref{thm:trivial}.

\begin{lemma}[Uniform Spectral Bound]
\label{lem:cov_bound}
For all iterations $t \ge 0$, the encoder noise covariance $\mathbf{\Sigma}_{\mathbf{W}}^{(t)}$ satisfies $\|\mathbf{\Sigma}_{\mathbf{W}}^{(t)}\|_2 \le 1$.
\end{lemma}

\begin{proof}
From the stationarity conditions in Lemma~\ref{lem:beta_opt}, the update for the encoder noise covariance is:
\begin{equation*}
\mathbf{\Sigma}_{\mathbf{W}}^{(t+1)} = \Bigl[ \mathbf{I}_m + \mathbf{A}^{(t)\top} \bigl(\mathbf{\Sigma}_{\mathbf{Z}}^{(t)-1}/\beta\bigr) \mathbf{A}^{(t)} \Bigr]^{-1}.
\end{equation*}
Since $\mathbf{\Sigma}_{\mathbf{Z}}^{(t)} \succ \mathbf{0}$ and $\beta > 1$, the term $\mathbf{M} = \mathbf{A}^{(t)\top}(\mathbf{\Sigma}_{\mathbf{Z}}^{(t)-1}/\beta)\mathbf{A}^{(t)}$ is positive semi-definite ($\mathbf{M} \succeq \mathbf{0}$). It follows that $(\mathbf{I}_m + \mathbf{M}) \succeq \mathbf{I}_m$. In the Loewner order, taking the inverse reverses the inequality:
\begin{equation*}
\mathbf{\Sigma}_{\mathbf{W}}^{(t+1)} = (\mathbf{I}_m + \mathbf{M})^{-1} \preceq \mathbf{I}_m.
\end{equation*}
This implies that all eigenvalues $\lambda_i(\mathbf{\Sigma}_{\mathbf{W}}^{(t+1)})$ lie in the interval $(0, 1]$. Consequently, the spectral norm satisfies \[\|\mathbf{\Sigma}_{\mathbf{W}}^{(t+1)}\|_2 = \lambda_{\max}(\mathbf{\Sigma}_{\mathbf{W}}^{(t+1)}) \le 1.\]
\end{proof}

\subsection{Proof of Theorem~\ref{thm:trivial}: Informational Collapse}

\begin{proof}
Let $(\mathbf{B}^{(t)},\mathbf{\Sigma}_{\mathbf{W}}^{(t)})$ denote the encoder parameters at iteration $t$. By substituting the coupled stationarity updates from Lemma~\ref{lem:beta_opt}, we derive the following recursive identity for the encoder gain matrix $\mathbf{B}$ after $n$ steps:
\begin{equation}
\label{eq:b_recursion}
\mathbf{B}^{(t+n)} = \beta^{-n} \mathbf{\Sigma}_{\mathbf{W}}^{(t+n)} \bigl[\mathbf{\Sigma}_{\mathbf{W}}^{(t)}\bigr]^{-1} \mathbf{B}^{(t)}.
\end{equation}

Taking the spectral norm on both sides of Eq.~\eqref{eq:b_recursion} and applying the sub-multiplicative property of the norm along with the uniform bound from Lemma~\ref{lem:cov_bound}:
\begin{equation*}
\|\mathbf{B}^{(t+n)}\|_2 \le \beta^{-n} \|\mathbf{\Sigma}_{\mathbf{W}}^{(t+n)}\|_2 \cdot \bigl\| [\mathbf{\Sigma}_{\mathbf{W}}^{(t)}]^{-1} \mathbf{B}^{(t)} \bigr\|_2 \le \beta^{-n} C_t,
\end{equation*}
where $C_t = \|[\mathbf{\Sigma}_{\mathbf{W}}^{(t)}]^{-1} \mathbf{B}^{(t)}\|_2$ is a constant determined by the state at iteration $t$. For the regularization regime $\beta > 1$, the term $\beta^{-n}$ vanishes exponentially as $n \to \infty$. This forces the spectral norm $\|\mathbf{B}^{(t)}\|_2 \to 0$, identifying the unique stable fixed point as the trivial solution $(\mathbf{A}=\mathbf{0}, \mathbf{B}=\mathbf{0}, \mathbf{\Sigma}_{\mathbf{Z}}=\mathbf{\Sigma}_{\mathbf{Y}},\mathbf{\Sigma}_{\mathbf{W}}=\mathbf{I}_m)$.

To evaluate the informational impact, we examine the joint covariance $\mathbf{\Sigma}_{(\mathbf{X}, \mathbf{V})}$. As established in the generative framework, the cross-covariance is $\mathbf{\Sigma}_{\mathbf{X}\mathbf{V}} = \mathbf{B}\mathbf{\Gamma}\mathbf{\Sigma}_{\mathbf{V}}$. In the limit $\mathbf{B} \to \mathbf{0}$, this block vanishes. The mutual information for this Gaussian system is:
\begin{equation*}
I(\mathbf{X};\mathbf{V}) = \frac{1}{2} \log \frac{\det(\mathbf{\Sigma}_{\mathbf{X}}) \det(\mathbf{\Sigma}_{\mathbf{V}})}{\det(\mathbf{\Sigma}_{(\mathbf{X},\mathbf{V})})}.
\end{equation*}
As $\mathbf{\Sigma}_{\mathbf{X}\mathbf{V}} \to \mathbf{0}$, the joint covariance matrix $\mathbf{\Sigma}_{(\mathbf{X},\mathbf{V})}$ becomes block-diagonal, such that $\det(\mathbf{\Sigma}_{(\mathbf{X},\mathbf{V})}) \to \det(\mathbf{\Sigma}_{\mathbf{X}})\det(\mathbf{\Sigma}_{\mathbf{V}})$. The ratio inside the logarithm approaches unity, driving $I(\mathbf{X};\mathbf{V}) \to 0$. This confirms that the spectral contraction of the encoder gain necessitates a total informational collapse.
\end{proof}

\section{Additional Qualitative Results}

This appendix presents supplementary qualitative results for the three nonlinear
datasets studied in the main text: dSprites, Shapes3D, and MPI3D-real. These
figures are intended to complement the quantitative evaluations reported in
Sections~\ref{sec:numerics_linear} and~\ref{sec:nonlinear} by illustrating typical
reconstruction behavior and latent-factor associations under selected
$(\beta,\lambda)$ configurations.

For each dataset, we show:
(i) representative original images alongside reconstructions, and
(ii) mutual information heatmaps between latent dimensions and ground-truth
generative factors.
All results are shown for a fixed random seed (seed~0) and are illustrative
rather than exhaustive.

\begin{figure*}[htb]
\centering

\begin{minipage}[t]{0.245\textwidth}\centering
  \includegraphics[width=\linewidth]{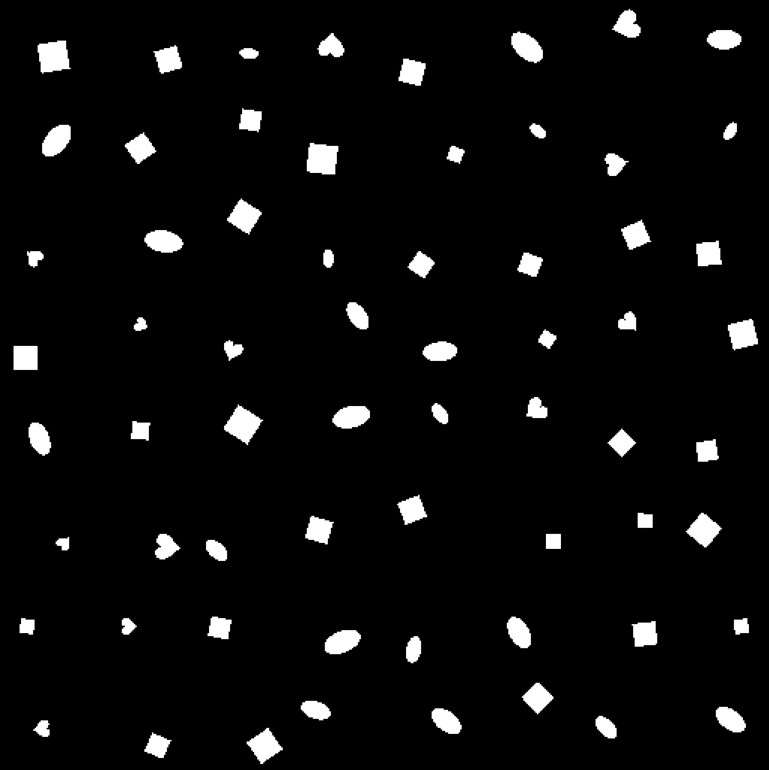}\\[-1mm]
  \small Original
\end{minipage}\hfill
\begin{minipage}[t]{0.245\textwidth}\centering
  \includegraphics[width=\linewidth]{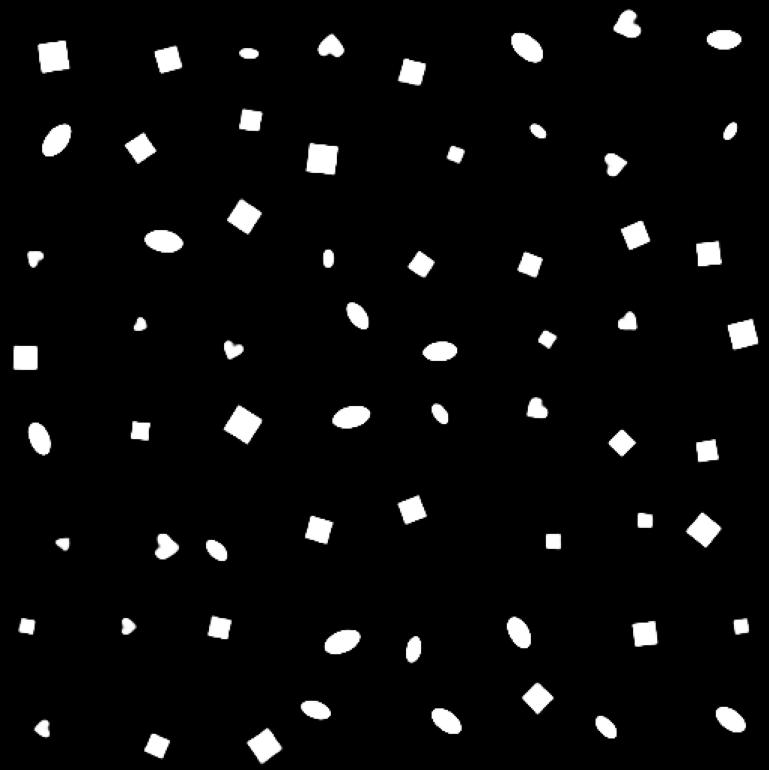}\\[-1mm]
  \small $\beta=1,\ \lambda=0$
\end{minipage}\hfill
\begin{minipage}[t]{0.245\textwidth}\centering
  \includegraphics[width=\linewidth]{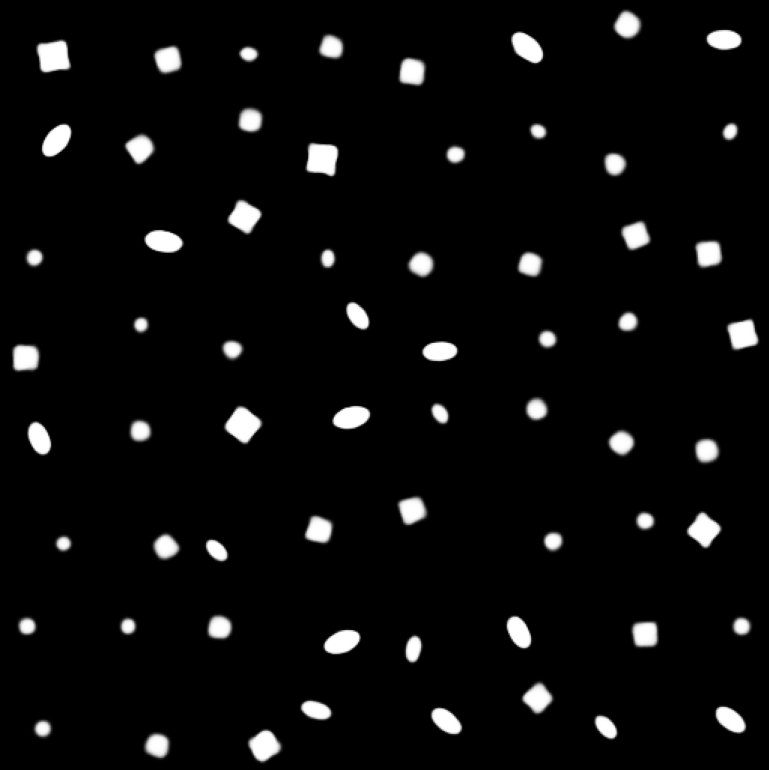}\\[-1mm]
  \small $\beta=8,\ \lambda=0$
\end{minipage}\hfill
\begin{minipage}[t]{0.245\textwidth}\centering
  \includegraphics[width=\linewidth]{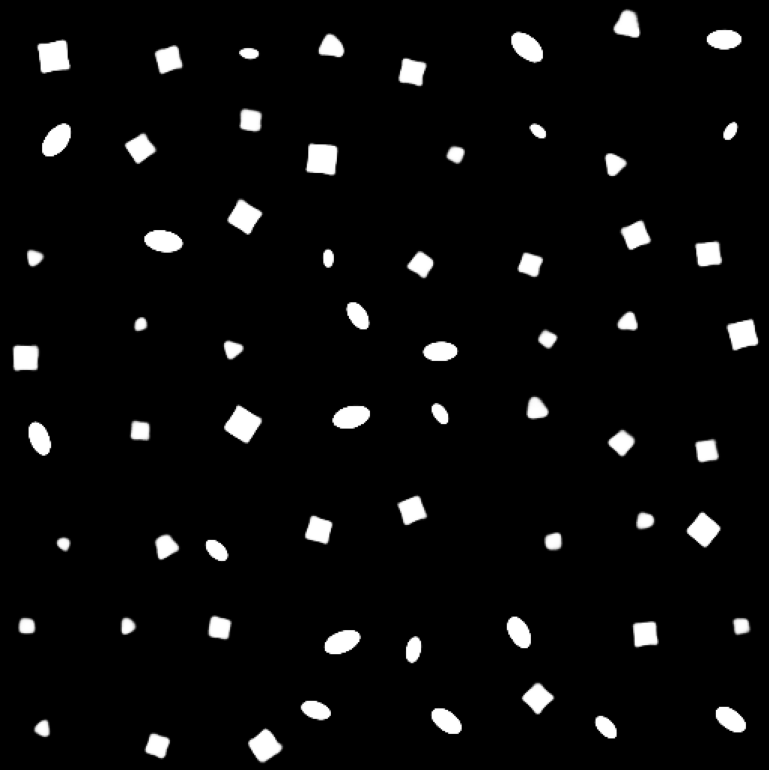}\\[-1mm]
  \small $\beta=8,\ \lambda=32$
\end{minipage}

\vspace{2mm}

\begin{minipage}[t]{0.245\textwidth}\centering
  \includegraphics[width=\linewidth]{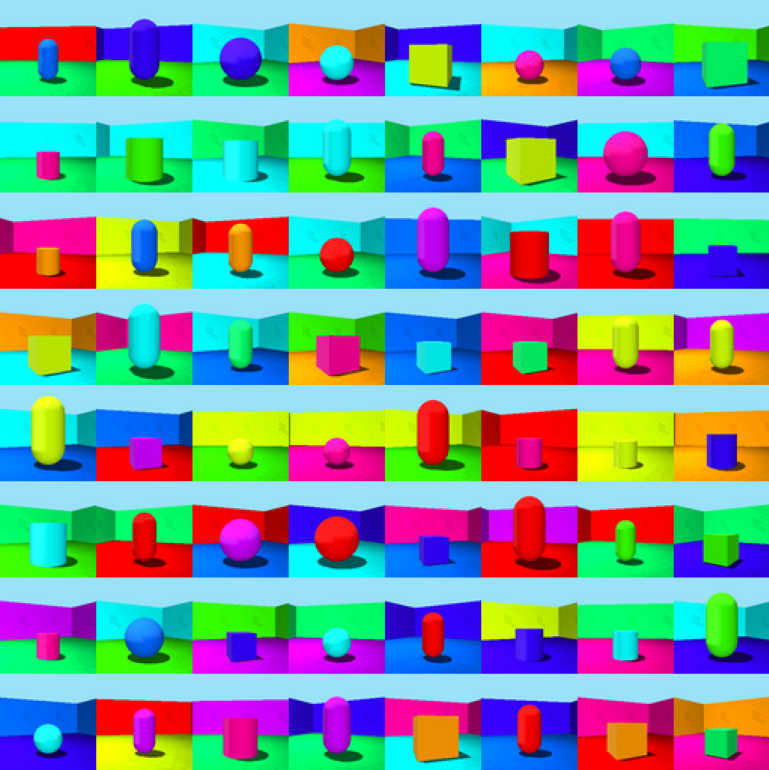}\\[-1mm]
  \small Original
\end{minipage}\hfill
\begin{minipage}[t]{0.245\textwidth}\centering
  \includegraphics[width=\linewidth]{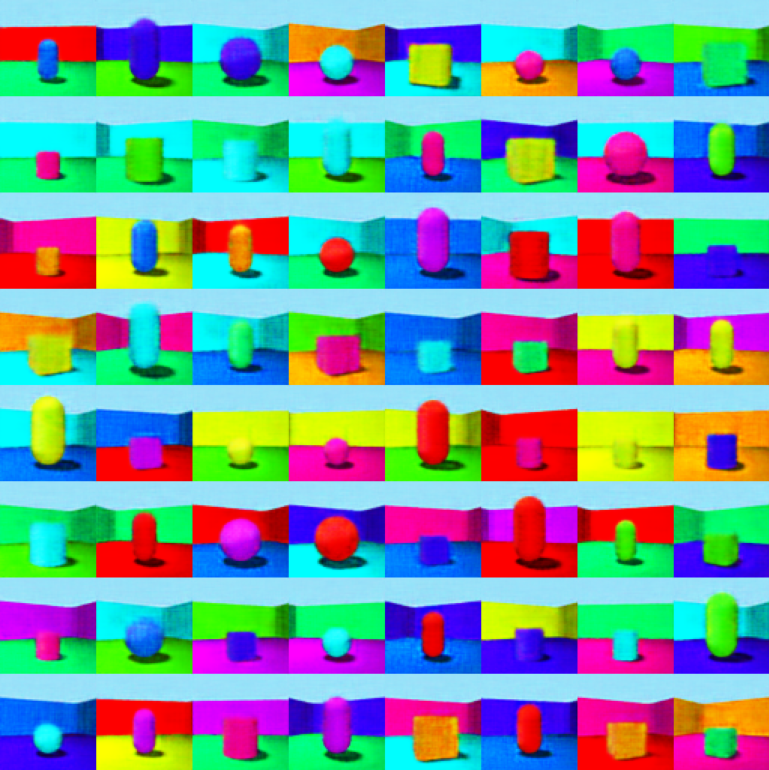}\\[-1mm]
  \small $\beta=1,\ \lambda=0$
\end{minipage}\hfill
\begin{minipage}[t]{0.245\textwidth}\centering
  \includegraphics[width=\linewidth]{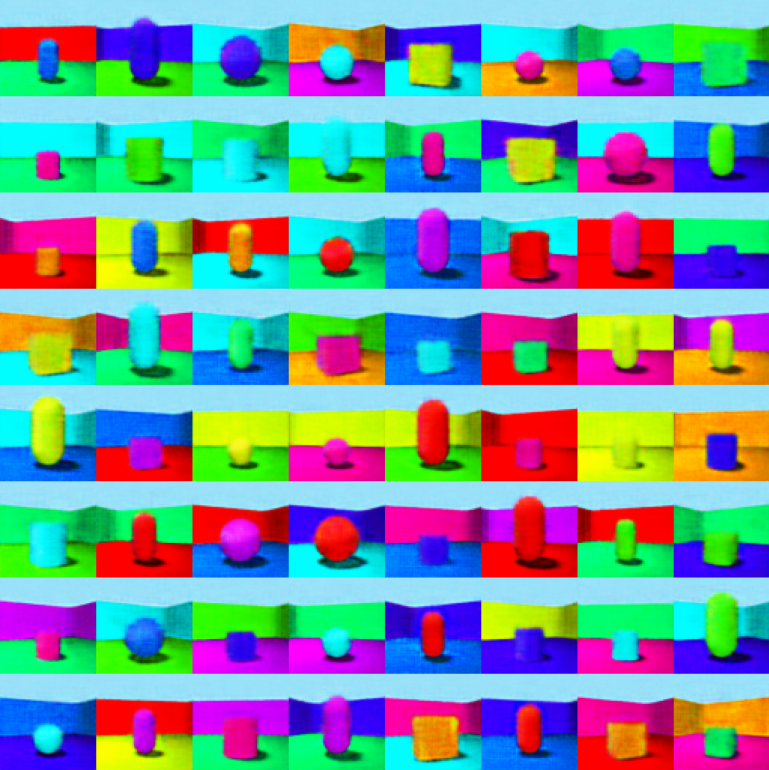}\\[-1mm]
  \small $\beta=8,\ \lambda=0$
\end{minipage}\hfill
\begin{minipage}[t]{0.245\textwidth}\centering
  \includegraphics[width=\linewidth]{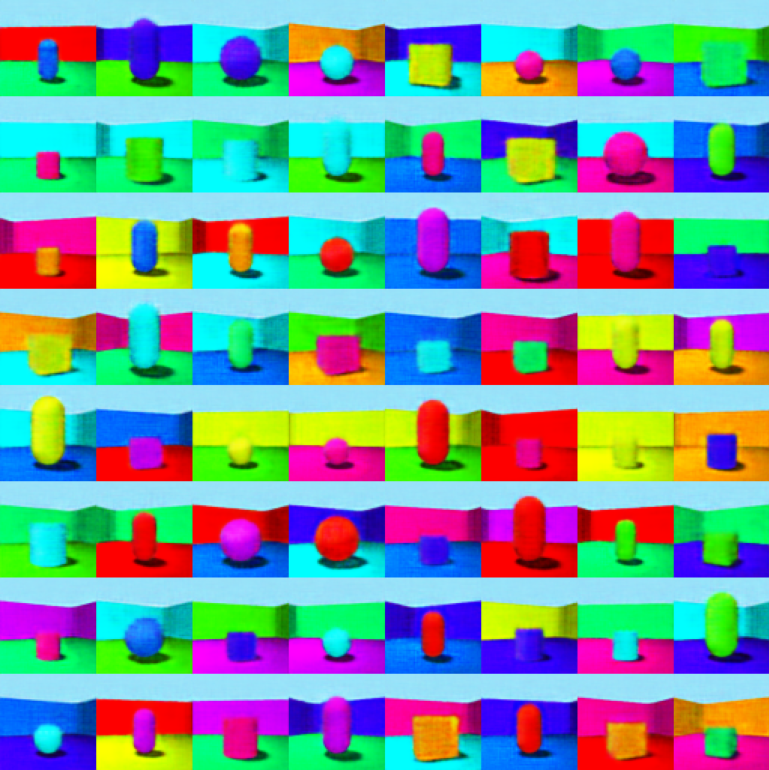}\\[-1mm]
  \small $\beta=8,\ \lambda=32$
\end{minipage}

\vspace{2mm}

\begin{minipage}[t]{0.245\textwidth}\centering
  \includegraphics[width=\linewidth]{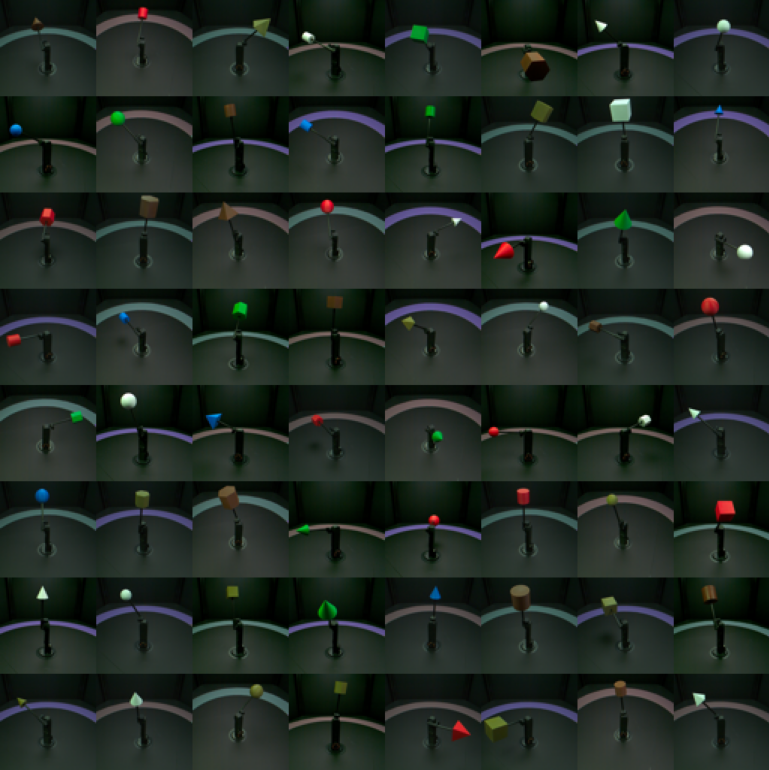}\\[-1mm]
  \small Original
\end{minipage}\hfill
\begin{minipage}[t]{0.245\textwidth}\centering
  \includegraphics[width=\linewidth]{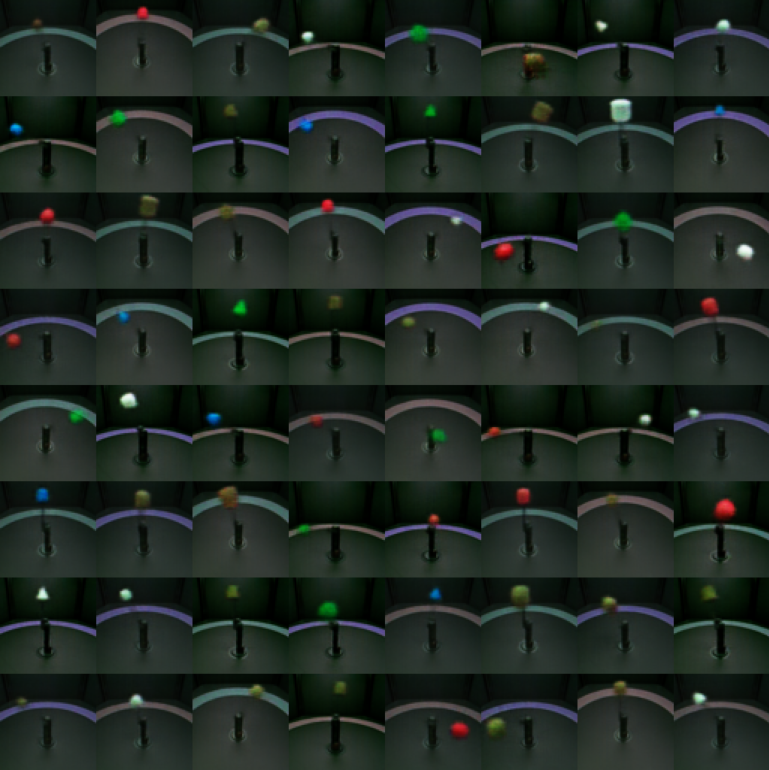}\\[-1mm]
  \small $\beta=1,\ \lambda=0$
\end{minipage}\hfill
\begin{minipage}[t]{0.245\textwidth}\centering
  \includegraphics[width=\linewidth]{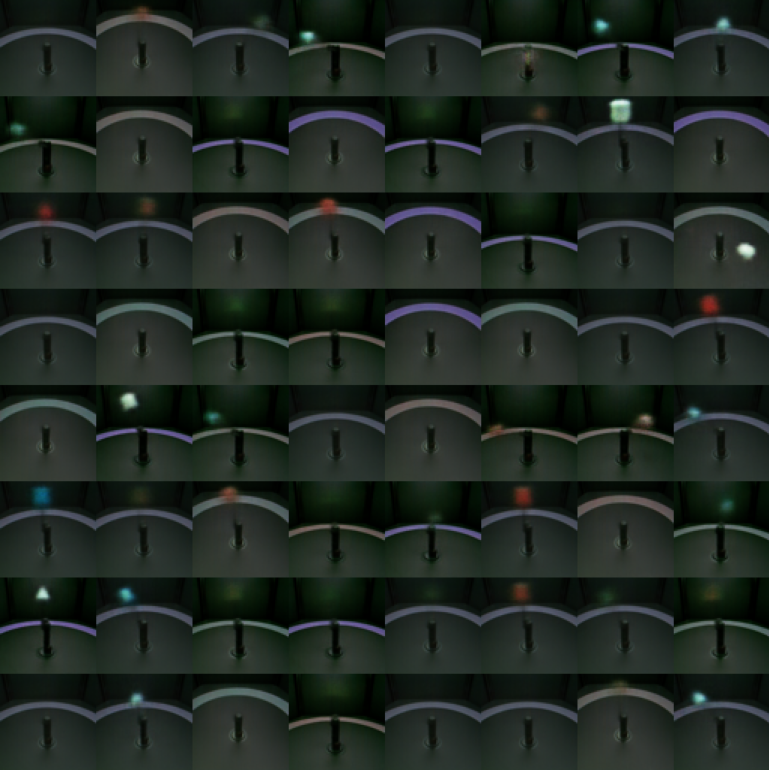}\\[-1mm]
  \small $\beta=8,\ \lambda=0$
\end{minipage}\hfill
\begin{minipage}[t]{0.245\textwidth}\centering
  \includegraphics[width=\linewidth]{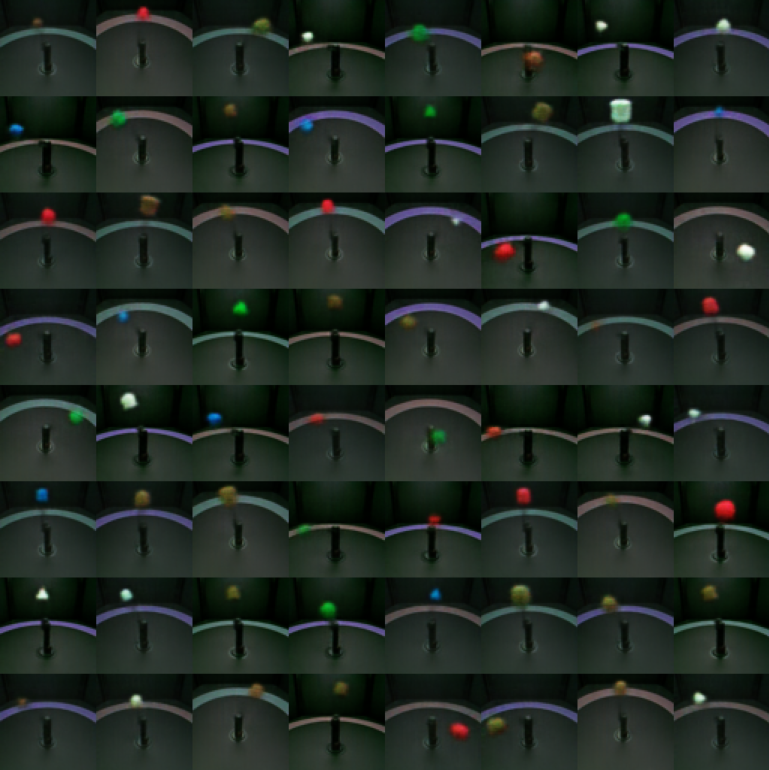}\\[-1mm]
  \small $\beta=8,\ \lambda=32$
\end{minipage}

\caption{\textbf{Qualitative reconstructions (seed 0).}
Each row shows original images and reconstructions under representative
$(\beta,\lambda)$ settings. At fixed $\beta$, introducing $\lambda>0$
substantially improves reconstruction fidelity. At $\lambda=0$, larger $\beta$
increases regularization pressure and can visibly degrade reconstructions.}
\label{fig:appendix_recons_all}
\end{figure*}

\begin{figure*}[htb]
\centering

\begin{minipage}[t]{0.5\textwidth}\centering
  \includegraphics[width=\linewidth]{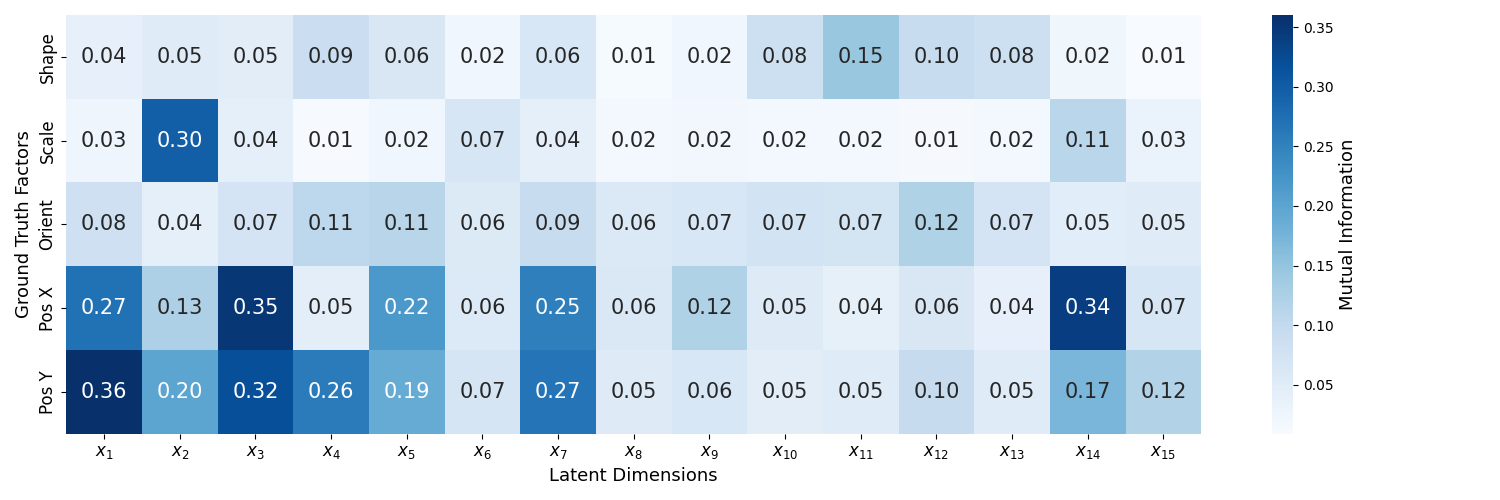}\\[-1mm]
  \small dSprites: $\beta=1,\ \lambda=0$
\end{minipage}\hfill
\begin{minipage}[t]{0.5\textwidth}\centering
  \includegraphics[width=\linewidth]{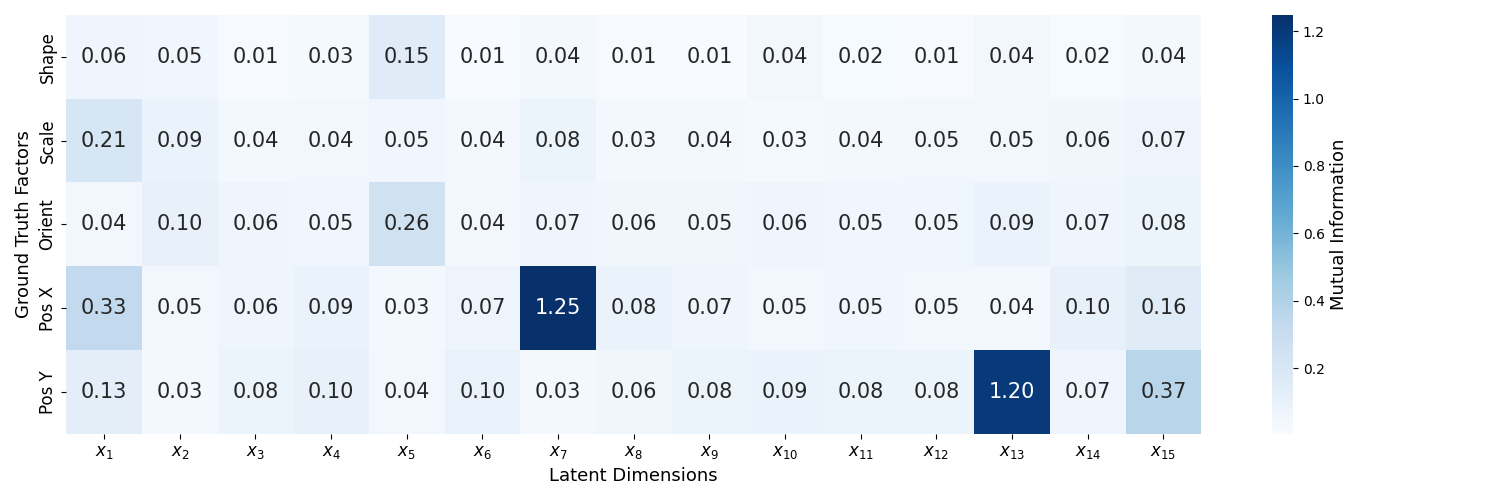}\\[-1mm]
  \small dSprites: $\beta=8,\ \lambda=32$
\end{minipage}

\vspace{2mm}

\begin{minipage}[t]{0.5\textwidth}\centering
  \includegraphics[width=\linewidth]{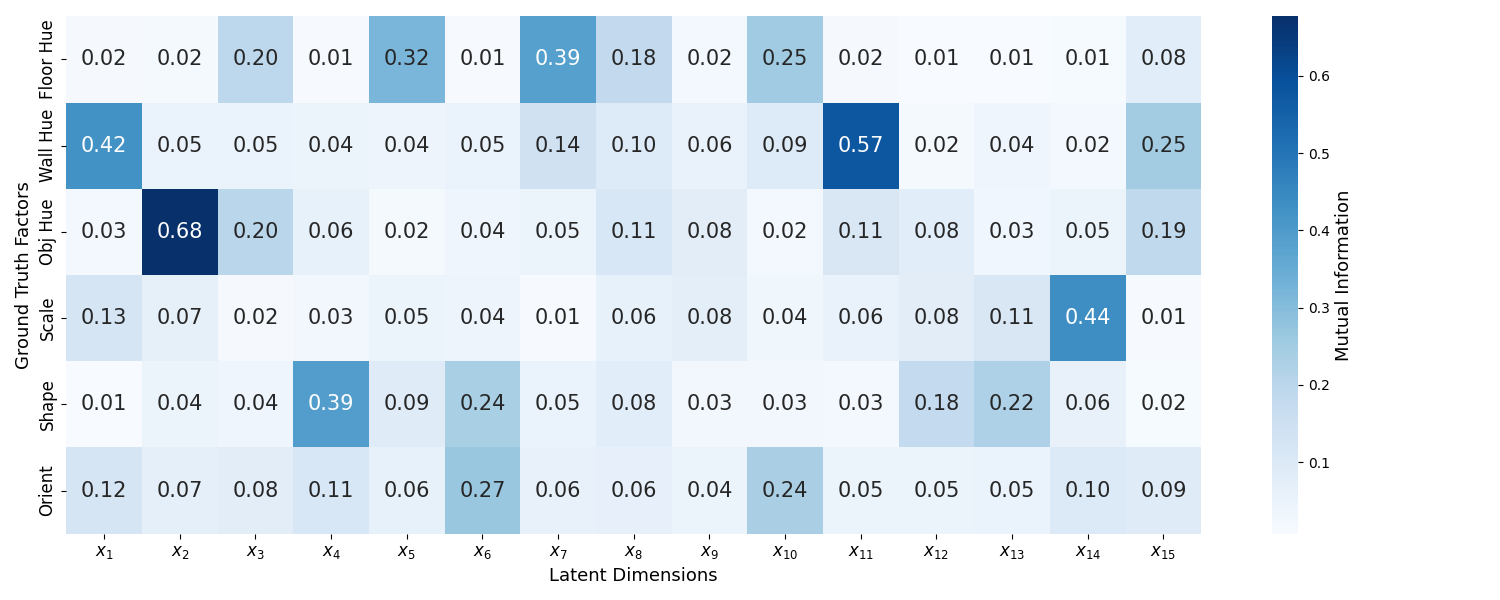}\\[-1mm]
  \small Shapes3D: $\beta=1,\ \lambda=0$
\end{minipage}\hfill
\begin{minipage}[t]{0.5\textwidth}\centering
  \includegraphics[width=\linewidth]{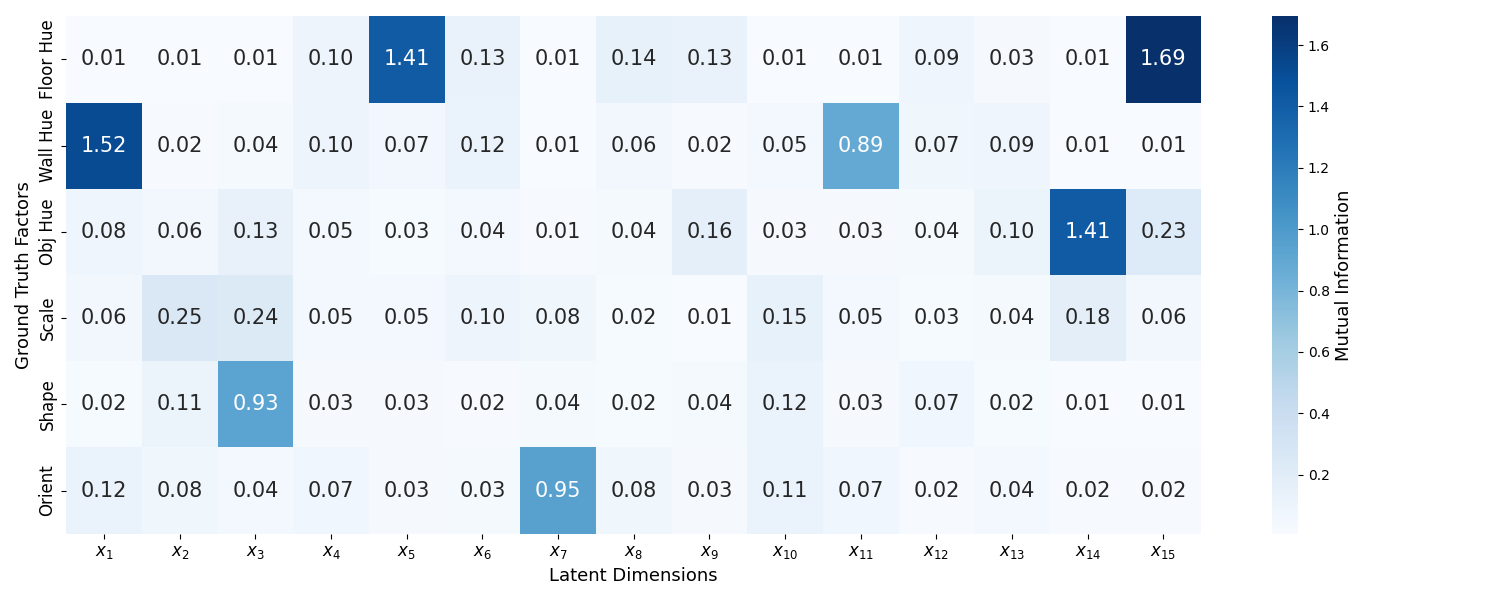}\\[-1mm]
  \small Shapes3D: $\beta=8,\ \lambda=32$
\end{minipage}

\vspace{2mm}

\begin{minipage}[t]{0.5\textwidth}\centering
  \includegraphics[width=\linewidth]{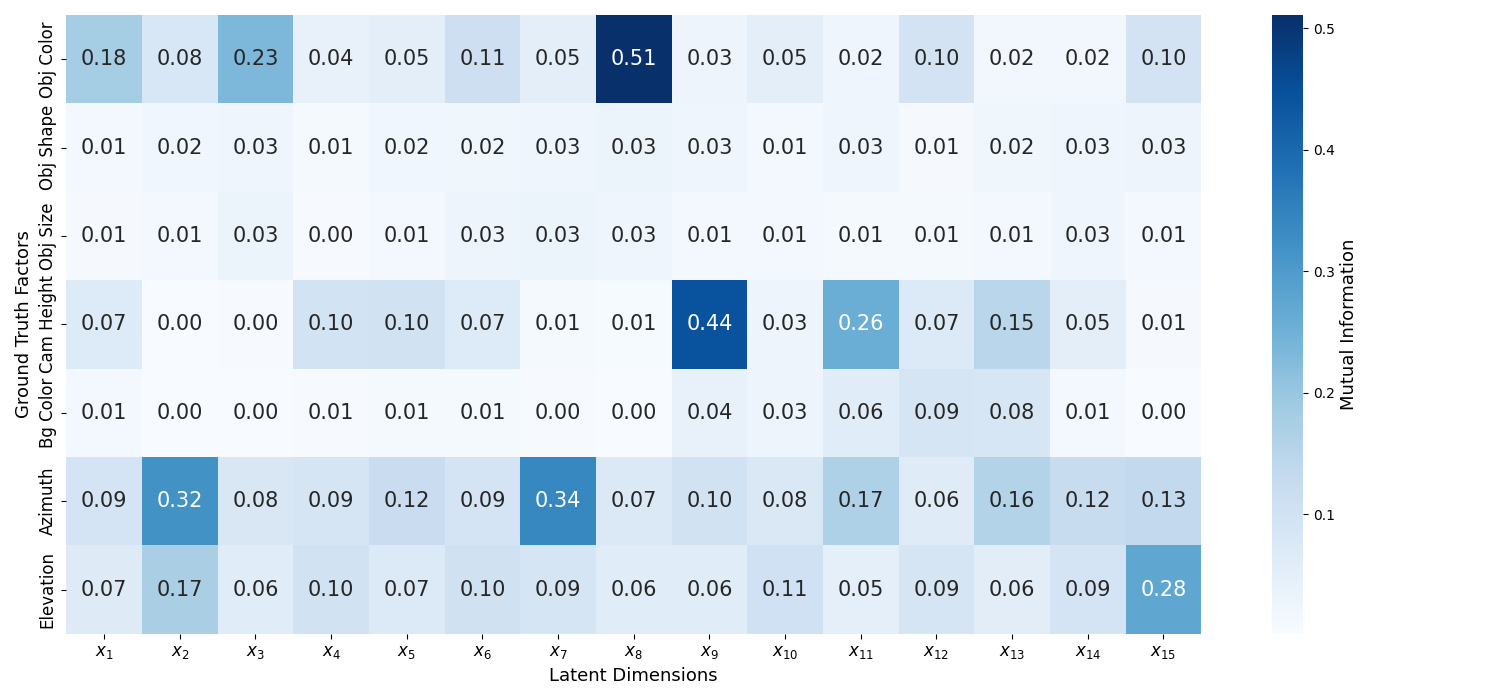}\\[-1mm]
  \small MPI3D-real: $\beta=1,\ \lambda=0$
\end{minipage}\hfill
\begin{minipage}[t]{0.5\textwidth}\centering
  \includegraphics[width=\linewidth]{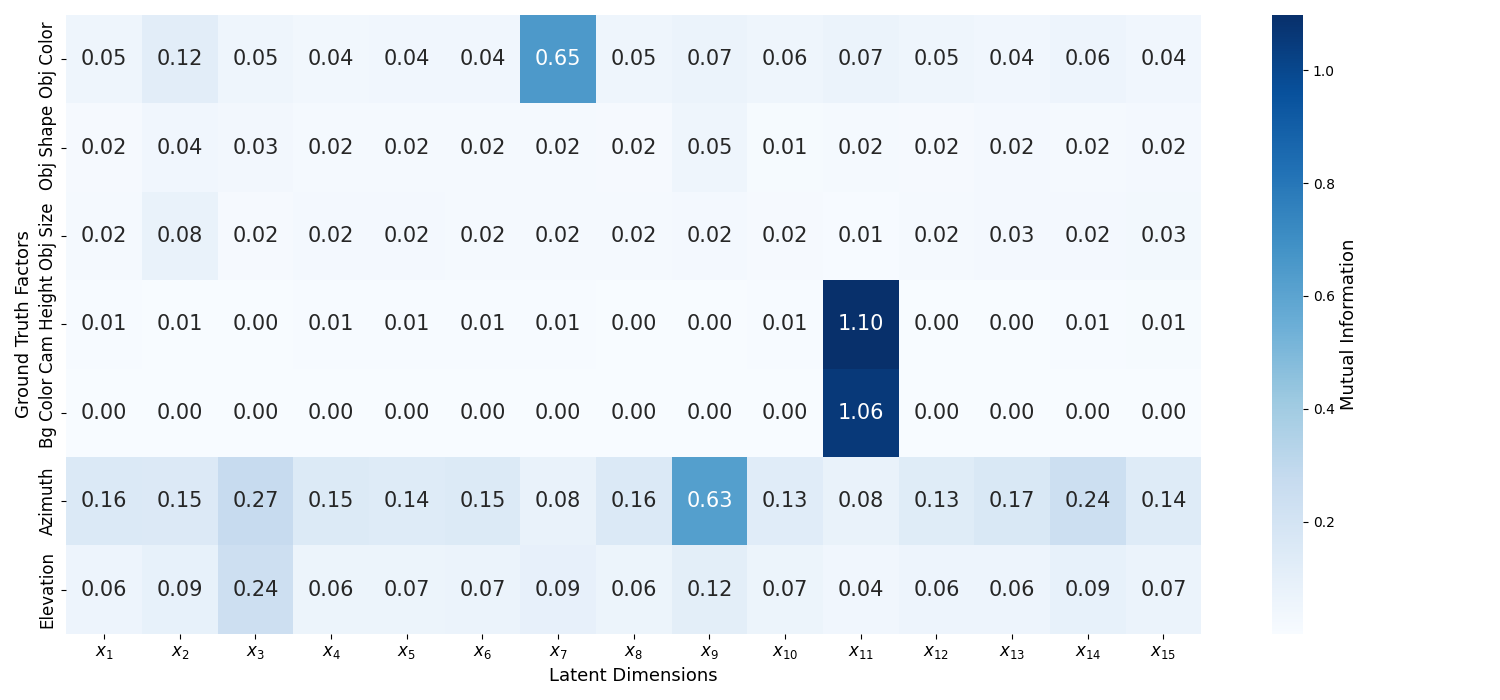}\\[-1mm]
  \small MPI3D-real: $\beta=8,\ \lambda=32$
\end{minipage}

\caption{\textbf{Mutual-information (MI) heatmaps (seed 0).}
Heatmaps show MI between latent coordinates and ground-truth factors under a
low-regularization setting ($\beta=1,\lambda=0$) and a higher-regularization
information-preserving setting ($\beta=8,\lambda=32$). The latter typically
exhibits more concentrated latent-factor associations (reduced redundancy),
consistent with improved disentanglement metrics reported in the main text.}
\label{fig:appendix_mi_all}
\end{figure*}

\end{document}